\documentclass[twoside,11pt]{article}

\usepackage{amsmath}        
\usepackage{amsthm}
\usepackage{mathtools}
\usepackage{caption}
\usepackage{subcaption}

\usepackage[preprint]{jmlr2e}
\usepackage{booktabs}       
\usepackage{microtype}      

\usepackage{tikz}
\usepackage{pgfplots}
\usepackage{pgfplotstable}
\pgfplotsset{compat=newest}
\usepackage{multirow}
\usepackage{multicol}
\usepackage{enumitem}
\setlist{nosep}
\usepackage{xspace}
\usepackage{lipsum}
\usepackage{xfrac}
\usepackage{scalerel}

\usepackage[cal=dutchcal, scr=boondox]{mathalfa}
\usepackage{dsfont}

\usepackage[algo2e, plain, vlined, linesnumbered]{algorithm2e}
\makeatletter
\renewcommand{\@algocf@capt@plain}{above}
\makeatother
\DontPrintSemicolon
\SetKwInOut{Input}{input}
\SetKwInOut{Output}{output}
\def\assign{\leftarrow}

\renewcommand{\O}{\mathscr{O}}
\newcommand{\Otilde}{\widetilde{\mathscr{O}}}
\newcommand{\zeroone}{\mbox{$0\hspace*{-2pt}-\hspace*{-3pt}1$}\xspace}

\newcommand{\ie}{\emph{i.e.\xspace}}

\newcommand{\Id}[1]{\mathds{1}\hspace{-2.5pt}\left[#1\right]}

\let\oldforall=\forall
\renewcommand{\forall}{\hspace{1pt}\oldforall\hspace{1pt}}

\let\oldexists=\exists
\renewcommand{\exists}{\hspace{1pt}\oldexists\hspace{1pt}}

\newcommand{\pr}[1]{\left(#1\right)}
\newcommand{\cb}[1]{\left\{#1\right\}}
\newcommand{\abs}[1]{\left|#1\right|}
\newcommand{\norm}[1]{\left\lVert#1\right\rVert}

\newcommand{\D}{\mathcal{D}}
\newcommand{\E}{\mathcal{E}}
\newcommand{\F}{\mathcal{F}}
\newcommand{\f}{\mathbf{f}}
\renewcommand{\H}{{\mathcal{H}}}
\newcommand{\K}{\mathcal{K}}
\newcommand{\N}{\mathcal{N}}
\newcommand{\x}{\mathbf{x}}
\newcommand{\X}{\mathcal{X}}
\newcommand{\Y}{\mathcal{Y}}
\newcommand{\fhat}{\hat{\f}}
\newcommand{\z}{\mathbf{z}}

\newcommand{\gammacover}{\frac{\gamma}{2^{2-1/p}}}
\newcommand{\fcover}{\widehat{\F}_T}

\newcommand{\reals}{\mathds{R}}
\newcommand{\naturals}{\mathds{N}}

\newcommand{\eqdef}{\overset{\smash{\mbox{\tiny \rm def}}}{=}}

\DeclareMathOperator*{\Prob}{\mathds{P}}
\newcommand{\prob}[2]{\Prob_{#1}\left[#2\right]}

\DeclareMathOperator*{\Expectation}{\mathds{E}}
\newcommand{\exv}[2]{\Expectation_{#1}\left[#2\right]}

\DeclareMathOperator*{\argmax}{argmax}

\newcommand{\Bin}{\textnormal{Bin}}
\newcommand{\BinInv}{\overline{\textnormal{Bin}}}

\newcommand{\hyp}{\textnormal{hyp}}
\newcommand{\Hyp}{\textnormal{Hyp}}
\newcommand{\HypInv}{\overline{\textnormal{Hyp}}}
\newcommand{\HypInvLower}{\underline{\textnormal{Hyp}}}

\jmlrheading{1}{2021}{N/A}{N/A}{N/A}{leboeuf21a}{Jean-Samuel Leboeuf and Frédéric LeBlanc and Mario Marchand}

\ShortHeadings{Hypergeometric Tail Inversion}{Leboeuf, LeBlanc and Marchand}
\firstpageno{1}

\begin{document}

\title{Improving Generalization Bounds for VC Classes\\ Using the Hypergeometric Tail Inversion}

\author{\name Jean-Samuel Leboeuf \email jean-samuel.leboeuf.1@ulaval.ca \\
        \addr Department of Computer Science\\
        Universit\'e Laval\\
        Qu\'ebec, QC, Canada
        \AND
        \name Fr\'ed\'eric LeBlanc \email flebl035@uottawa.ca\\
        \addr Department of Mathematics and Statistics\\
        University of Ottawa\\
        Ottawa, ON, Canada
        \AND
        \name Mario Marchand \email mario.marchand@ift.ulaval.ca \\
        \addr Department of Computer Science\\
        Universit\'e Laval\\
        Qu\'ebec, QC, Canada
        }

\editor{TBD}

\maketitle

\begin{abstract}
We significantly improve the generalization bounds for VC classes by using two main ideas.
First, we consider the hypergeometric tail inversion to obtain a very tight non-uniform distribution-independent risk upper bound for VC classes.
Second, we optimize the ghost sample trick to obtain a further non-negligible gain.
These improvements are then used to derive a relative deviation bound, a multiclass margin bound, as well as a lower bound.
Numerical comparisons show that the new bound is nearly never vacuous, and is tighter than other VC bounds for all reasonable data set sizes.
\end{abstract}


\section{Introduction}


In their pioneering work, \citet{vapnik71} identified a key combinatorial property of classes of functions, now called the Vapnik-Chervonenkis (VC) dimension, that characterizes the generalization properties of any learning algorithm operating on classes whose VC dimension is finite.
They showed that the difference between the true and empirical risks of a classifier is bounded from above in $\O\big(\!\pr{\frac{d}{m} \log \frac{m}{d}}^{\scriptscriptstyle 1/2}\big)$
with high probability, where $m$ is the sample size and $d$ the VC dimension of the hypothesis class---a result known as \emph{Vapnik's pessimistic bound}.
A second result of \citet{vapnik71} (improved by \citet{as-93} and \citet{vapnik98}), called \emph{Vapnik's relative deviation bound}, upper bounds the difference between the true and empirical risks of a classifier in $\smash[t]{\Otilde\big(\frac{d+\sqrt{dk}}{m}\big)}$%
\footnote{We use the $\Otilde$ notation to specify that $\log\frac{m}{d}$ factors are ignored.}%
, where $k$ is the number of errors made by the classifier on the sample.
An advantage of this bound is that it interpolates correctly between the asymptotic behaviors of the agnostic ($k\ne 0$) and the realizable ($k=0$) cases (the latter is known to be in $\widetilde{\Theta}\big(\frac{d}{m}\big)$), albeit at the cost of a larger universal constant than the pessimistic bound.

In parallel to these bounds, a complementary result showed that generalization bounds based on VC dimension were also bounded from below in $\Omega\big(\sqrt{d/m}\big)$, \ie, without the logarithmic factor (see \citet{vapnik98}, \citet{as-93}, recently improved to optimality by \citet{kontorovich2019exact}).
Not long after, some upper bounds making use of \citet{dudley1978central}'s chaining technique successfully matched the asymptotic behavior of the lower bound (\citet{li2001improved,lugosi2002pattern}).

Rather than discouraging the community to pursue this line of work, these minimax results incited researchers to find new ways to improve the convergence rate of generalization bounds.
These so-called ``fast rate'' bounds enjoy a favorable asymptotic behavior in $\smash[t]{\Otilde\big(\frac{1}{m}\big)}$ at the expense of generality.
They all rely on some assumptions about the nature of the data labeling distribution, such as \emph{Tsybakov's margin condition} \citep{tsybakov04optimal,massart06risk,zhivotovskiy18localization}, the \emph{Bernstein condition} \citep{bartlett06empirical} and the \emph{central condition on proper learning} \citep{vanerven15fast}, or on some distribution-dependent quantities like local Rademacher complexities \citep{bartlett2002localized} and local VC dimension \mbox{\citep{oneto2016local}}.
Some have even been able to obtain exact results on the generalization capabilities of VC classes, such as \citet{vayatis2003exact}, \citet{vorontsov2010exact} and \citet{bottou2015making}.
Even though these bounds can be very powerful in the right situation, they suffer from a limited scope of application---either because the assumptions made on the unknown generating distribution are hard to verify, or because they make use of quantities hard to compute or approximate (such as Rademacher complexities) \citep{bartlett2002rademacher}.

While the community was mainly concerned with improving the convergence rate of generalization bounds, few actually tried to further decrease the constant factor of the original bounds of Vapnik.
Yet, it would be of great interest to have general distribution-independent bounds that are tight for real-life situations, where the sample size is relatively small and the distribution is unknown.
Indeed, to see that constant factors matter, one can show fairly easily (see Appendix~\ref{app:comparison_lugosi}) that \citet{lugosi2002pattern}'s $\Theta(\sqrt{d/m})$ chaining bound, which relies on \citet{haussler1995sphere}'s discretization trick, becomes tighter than \citet{vapnik98}'s pessimistic bound only when the sample size is on the order of $m \gtrsim 10^{335}d$. 
Hence, even if it was possible to reduce the constant factor of \citeauthor{lugosi2002pattern}'s bound, it is unlikely that this would lead to a tighter bound than Vapnik's bound for reasonable data set sizes.

In fact, the work of \citet{catoni2004improved} is the only one, to the best of our knowledge, to provide a meaningful improvement to the constant factors of Vapnik's bounds.
\citeauthor{catoni2004improved} achieves this by introducing the concept of ``partially exchangeable'' prior distributions on multiple ``shadow'' copies of the original sample in the transductive setting.
Thus, it feels as if bounds for VC classes are somewhat outdated compared to other types of bounds such as sample compression bounds \citep{floyd95sample} or PAC-Bayesian bounds \citep{mcallester1999some}, which have received a fair amount of attention over the last years.
However, contrary to the case of sample compression, where \citet{langford05} and \citet{Laviolette09} have shown that the bounds are nearly as tight as one might hope, distribution-independent bounds for VC classes still have room for improvements.

While it is true that several recent research avenues indicate that the VC dimension is not the relevant complexity parameter to characterize the generalization error of over-parametrized deep neural networks (see for examples \citet{bartlett2019nearly} and \citet{zhang2021understanding}), neurals networks suffer from a lack of interpretability.
This shortcoming often prevents them from being used in areas where human safety, health, or well-being is at stake.
Interpretable classifiers most often take the form of hard threshold decisions such as those provided by Boolean formulas (for example DNFs), decision lists, decision trees, rule sets, etc.
For these models, the VC dimension is relevant and can be computed (or approximated), and a generalization upper bound can guide learning algorithms to achieve the proper bias-complexity trade-off, as well as make them certifiable.

In light of these observations, we derive a tight non-uniform distribution-independent generalization bound for VC classes by exploiting two new ideas that have appeared since Vapnik's original work.
The proofs of Vapnik's bounds rely on two critical steps, called the ghost sample trick (which replaces the dependence on the true risk by a dependence on the risk of a ``ghost sample'') and the symmetrization trick (which allows simplifications by mixing the real sample with the ghost sample).
Taking inspiration from \citet{langford05}'s binomial tail inversion bound, we improve on the symmetrization step by considering the \emph{hypergeometric tail inversion}.
Then, we optimize the size of the ``ghost sample''---a step neglected by \citet{vapnik98} and \citet{as-93}---to obtain noticeable gains, in agreement with the observations of \citet{catoni2004improved}.
The new technique is general enough to also apply to a relative deviation bound, a margin bound, and a lower bound.

These improvements allow us to make very few approximations in the derivation of the bound; as such, the new bound is significantly tighter than Vapnik's bounds, Catoni's bound, and Lugosi's chaining bound for various practical sets of parameters.
It is asymptotically optimal up to logarithmic factors for both the realizable case and the agnostic case.
Furthermore, as we show with numerical comparisons, it is almost never vacuous, even for a sample size as small as 100 examples or for large risks, contrary to the previous bounds.

The paper is divided as follows.
In the next section, we present the notation, definitions, and the mathematical setting.
The third section is dedicated to the main result, where we give the proof of the new upper bound and analyze its strengths and weaknesses in depth.
We also discuss the trade-off offered by the ghost sample size.
We then briefly present a relative deviation bound, a margin bound, and a lower bound derived following the same improvements.
We conclude by numerically comparing the multiple bounds in order to be able to appreciate the actual gain offered by the new bound.

\section{Definitions and Notation}

Throughout this paper, we consider the multi-class classification problem where $\X$ denotes the instance space and $\Y = [n]$ for $n \in \naturals$, with $[n] \eqdef \{1,2,\dots,n\}$, denotes the label space, although we will assume $n=2$ in practical applications.
We also define $S \eqdef \big( (\x_1, y_1), \dots, (\x_m, y_m) \big) \in (\X\times \Y)^m$ as a collection of $m$ independent and identically distributed (i.i.d.) examples sampled according to some unknown probability distribution $\D$ on $\X\times\Y$.
Let $\H \subseteq [n]^{\X}$ be some hypothesis class.
We consider the \zeroone loss function, from which we define the true and empirical risk of a classifier $h\in\H$ as
\begin{equation*}
  R_\D(h) \eqdef \exv{(\x,y)\sim \D}{\Id{h(\x) \ne y}} \qquad \textnormal{and} \qquad
  R_S(h) \eqdef \frac{1}{m} \sum_{i=1}^m \Id{h(\x_i) \ne y_i},
\end{equation*}
where $\Id{\cdot}$ denotes the indicator function.
We adopt a notation where, given a random variable $Z$ following some distribution, $\exv{Z}{f(Z)}$ denotes the expectation of $f(Z)$ and $\prob{Z}{A(Z)}$ denotes the probability that event $A(Z)$ occurs.

To produce generalization bounds, VC theory makes use of the \emph{growth function} to characterize a given hypothesis class $\H$.
It is defined by $\tau_\H(m) \eqdef \max_{S:\left|S\right|=m} \left| \left\{ h|_S : h \in \H\right\} \right|$, where $h|_S \eqdef (h(\x_1), h(\x_2), \dots, h(\x_m))$, for $(\x_j,y_j) \in S$.
Note that this definition makes sense for any number of classes.
In the binary classification setting, one can use the \emph{Vapnik-Chervonenkis (VC) dimension} to approximate $\tau_\H$.
We define the VC dimension of $\H$ as the largest integer $d$ such that $\tau_\H(d) = 2^d$.
This quantity can be used to upper bound the growth function by Sauer-Shelah's lemma \citep{vapnik71}, which implies that $\tau_\H(m) \le \left( \frac{em}{d} \right)^d$ whenever $m \ge d$.
In the case where $n > 2$, one can use the Natarajan dimension \citep{bendavid95characterizations} to characterize $\H$ instead of the VC dimension so that a similar result holds for the growth function.

The main results of this paper rely on the hypergeometric distribution and its inverse.
We only give the essential definitions here; extensive detail is provided in Appendix~\ref{app:hypergeometric_distribution}.

A \emph{hypergeometric experiment} is the event of drawing $k$ successes in a sample of $m$ elements from a population of $M$ elements containing $K$ successes.
The \emph{hypergeometric distribution} gives the probability of such an event, which is equal to $\hyp(k,m,K,M) \eqdef \binom{K}{k}\binom{M-K}{m-k} / \binom{M}{m}$.
We refer to the cumulative distribution function as the \emph{hypergeometric tail function}, which we denote by $\Hyp(k,m,K,M) \eqdef \sum_{j=0}^k \hyp(j,m,K,M)$.
Figure~\ref{fig:hyp_tail_plot} of Appendix~\ref{app:hyp_tail} presents the hypergeometric tail function in terms of its four parameters.
Our main result requires us to invert the tail seen as a function of the parameter $K$.
Since $K$ is a discrete variable, the hypergeometric tail cannot be inverted straightforwardly.
Instead, one must use a pseudo-inverse, defined below.

\begin{definition}[Hypergeometric tail pseudo-inverse]\label{def:hyp_tail_inv}
We define the pseudo-inverse of the hypergeometric tail function, for any $k < m$ and any $\delta \in (0,1)$, as
\begin{equation*}
    \HypInv(k, m, \delta, M) \eqdef \min \cb{ K : \Hyp(k, m, K, M) \le \delta }.
\end{equation*}
\end{definition}
This choice of pseudo-inverse implies that $\Hyp\big(k,m,\HypInv(k,m,\delta,M),M\big) \le \delta$, a fact which will be useful later on.
Figure~\ref{fig:hyp_tail_plot_inv} in Appendix~\ref{app:hyp_tail_inv} plots the pseudo-inverse as a function of its four parameters.
We also provide two algorithms to compute this pseudo-inverse in Appendix~\ref{app:computing_HypInv}, accompanied by an implementation in Python available on GitHub\footnote{\url{https://github.com/jsleb333/hypergeometric_tail_inversion}}.
Note that pseudo-inverses are generally not unique; in fact, \citet{le2016validation} also introduce an similar but slightly different pseudo-inverse for the hypergeometric tail in the context of validation of matching algorithms.

\section{The Hypergeometric Tail Inversion Bound}
\label{sec:main_bound}

The new technique takes inspiration from \citet{langford05}'s binomial tail inversion theorem, which states that with probability $1-\delta$ on the choices of sample $S$ of size $m$, the true risk of a classifier $h$ is bounded from above by $\BinInv(mR_S(h),m,\delta)$, where $\BinInv$ is the binomial tail inverse (defined in Appendix~\ref{app:langfords_binomial_tail_inversion_theorem}, similarly to $\HypInv$).
The proof given by Langford feels more like an intuitive argumentation about why it should be true rather than an actual proof.
Therefore, we allow ourselves to present, in Appendix~\ref{app:langfords_binomial_tail_inversion_theorem}, a new formal proof of the theorem, which has in fact inspired this work.
The binomial tail inversion bound is ``essentially perfectly tight'' since no approximation is made in deriving it, with a major caveat: it is a test bound, meaning it does not hold simultaneously for all classifiers in a hypothesis class $\H$.
While one can use the union bound to make it hold for all $h\in\H$, this strategy is limited to the cases where $\H$ is countable.

In this section, we broaden the scope of \citeauthor{langford05}'s strategy by generalizing it to hypothesis classes with polynomial growth function.
We first state and explain the bound, and then prove it in detail.
Next, we examine the improvements over the traditional derivations of \cite{vapnik98} and \cite{as-93}, and analyze the weaker steps of the proof.
Finally, we discuss the effect of the ghost sample size and its implications.

\subsection{The Main Bound}

Typical \emph{uniform} bounds require that the absolute difference between the true risk $R_\D(h)$ and the empirical risk $R_S(h)$ is bounded by a function $\epsilon(m, \delta)$, which depends only on the number $m$ of examples and the confidence parameter $\delta$.
On the other hand, \emph{non-uniform} bounds allow the upper bound $\epsilon$ to depend on the particular classifier $h$ in addition to $m$ and $\delta$.
Both approaches have in common that they bound $\abs{R_\D(h)-R_S(h)}$.
This can provide interesting theoretical results, but in practice we are mainly concerned with finding an upper bound on the true risk.

The upper bound part of a uniform bound takes the form $R_\D(h) \le R_S(h) + \epsilon(m, \delta)$.
The linear dependence on the empirical risk being arbitrary, we lift this constraint and we look for a bound of the general form $R_\D(h) \le \epsilon(R_S(h), m, \delta)$ instead.
This transforms the bound into a non-uniform one, since $\epsilon$ now depends on $h$, but only via its empirical risk---the only dependence on the data that we allow ourselves.
As such, our main theorem, stated below, proposes a kind of middle ground between the uniform and non-uniform settings, similar to what \citeauthor{vapnik98}'s relative deviation bound offers.

\begin{theorem}[Hypergeometric tail inversion bound]
\label{thm:main}
Let $\H \subseteq \mathcal{Y}^\mathcal{X}$ be a hypothesis class.
Then for any distribution $\D$ on $\mathcal{X} \times \mathcal{Y}$, any confidence parameter $\delta \in (0, 1)$ and any integers $m, m'>0$, the following inequality holds:
\begin{equation*}
\prob{S\sim\D^m}{\forall h \in \H, R_\D(h) \leq \epsilon(m R_S(h), m, \delta)} > 1 - \delta,
\end{equation*}
where $\epsilon(k,m,\delta) = 1$ if $k=m$ and otherwise
\begin{equation*}
\epsilon(k, m, \delta) = \frac{1}{m'} \max\cb{ 1,\; \HypInv\big(k, m, \textstyle\frac{\delta}{4 \tau_\H(m + m')}, m + m'\big) - 1 - k }.
\end{equation*}
\end{theorem}
For convenience, we let $\epsilon$ depend on the number $m R_S(h)$ of errors made by the classifier $h$ instead of its empirical risk.
It is necessary to set $\epsilon(m,m,\delta)=1$, since $\HypInv$ is ill-defined for $k=m$.
We discuss the optimal choice of the parameter $m'$ in Section~\ref{ssec:ghost_sample_size_trade-off}.

The theorem is stated such that it holds for multi-class learning problems, as long as one possesses an upper bound for the growth function.
Moreover, the new bound is non-trivial only if the growth function is polynomial in the number $m$ of examples (which is equivalent to ask that the VC dimension be finite in the case of binary classification).
\smallskip

\begin{proof}
To prove the theorem, it is sufficient to show that $\prob{S}{\exists h \in \H, R_\D(h) > \epsilon(k_S(h))} < \delta$, where we mute some dependencies to alleviate the notation and we use $k_S(h)\eqdef m R_S(h)$ hereafter.
The left-hand side can be written as
\begin{equation}\label{eq:main_1}
\prob{S}{\exists h \in \H, R_\D(h) > \epsilon(k_S(h))} = \exv{S}{\sup_{h\in \H} \Id{R_\D(h) > \epsilon(k_S(h))} }.
\end{equation}

First, we proceed with the ghost sample trick as usual to transfer the dependence on the true risk $R_\D(h)$ onto the empirical risk $R_{S'}(h)$ of a ghost sample $S'$ consisting of $m'$ i.i.d.\ examples drawn from $\D$.
Fix some $h\in\H$ and notice that $k_{S'}(h) \eqdef m' R_{S'}(h)$ is a random variable following a binomial distribution with parameters $m'$ and $R_\D(h)$.
Hence, we can use the inequality
\begin{equation}
\prob{S'}{ m'R_{S'}(h) \geq m'R_\D(h)} > \frac{1}{4}, \label{eq:ghost_sample}
\end{equation}
which holds provided that $R_\D(h) > 1 / m'$, as proven by \cite{greenberg14} (but claimed to be true by \cite{as-93}).
Since we have that $\epsilon(k) \geq 1/m'$ according to its definition in the theorem, it is sufficient to assume that $R_\D(h) > \epsilon(k_S(h))$ to have
\begin{equation*}
\prob{S'}{ m' R_{S'}(h) > m' \epsilon(k_S(h)) } \geq \prob{S'}{ m' R_{S'}(h) \geq m' R_\D(h) } > \frac{1}{4}.
\end{equation*}
This implies that for all $h\in\H$ and all $S$, we have
\begin{equation*}
\Id{R_\D(h) > \epsilon(k_S(h))} < 4 \prob{S'}{ m' R_{S'}(h) > m' \epsilon(k_S(h))}.
\end{equation*}

Inserting this result inside Equation~\eqref{eq:main_1} and then using the fact that the supremum of an expectation is less than or equal to the expectation of the supremum, one has
\begin{align}
\prob{S}{\exists h \in \H, R_\D(h) > \epsilon(k_S(h))}
    &< 4 \exv{S}{\sup_{h\in \H} \exv{S'}{\Id{k_{S'}(h) > m'\epsilon(k_S(h))}} }\nonumber\\
    &\le 4\! \exv{S, S'}{\sup_{h\in \H} \Id{k_{S'}(h) > m'\epsilon(k_S(h))}}.\label{eq:main_2}
\end{align}

Now, let $T \eqdef (S, S') = ((\x_1,y_1), \dots, (\x_{m+m'}, y_{m+m'}))$ be an i.i.d.\ sample constructed by re-indexing $S$ and $S'$, and let $\H|_T = \cb{(h(\x_1),\dots,h(\x_{m+m'})) : h \in \H}$
be the restriction of $\H$ to $T$.
Furthermore, let $\sigma \in \Sigma \eqdef \{ \sigma \subseteq [m + m'] \: : \: \abs{\sigma} = m \}$ be some set of $m$ indices taken from $[m+m']$, and denote its elements by $\sigma_1$, $\sigma_2$, ..., $\sigma_m$ such that they are in increasing order.
We define $T(\sigma)\eqdef((\x_{\sigma_1}, y_{\sigma_1}), \dots, (\x_{\sigma_m}, y_{\sigma_m}))$ to be the collection of examples in $T$ whose indices are in $\sigma$.
With this notation, we have that $S = T([m])$, and the right-hand side of Inequality~\eqref{eq:main_2} becomes
\begin{align*}
4\! \exv{S, S'}{\sup_{h\in \H} \Id{k_{S'}(h) > m'\epsilon(k_S(h))}}
    \!= 4 \exv{T}{\max_{h\in \H|_T} \Id{ k_{T}(h) > m'\epsilon(k_{T([m])}(h)) + k_{T([m])}(h)}}.
\end{align*}

We are then ready to apply the symmetrization trick.
Because the examples of $T$ are i.i.d.\ and because we are taking the expectation over $T$, the right-hand side of the last equation would be the same if one were to replace $T([m])$ with $T(\sigma)$ for any set $\sigma \in \Sigma$.
Thus, we can take the expectation of this expression over all possible sets $\sigma\in\Sigma$ with the uniform probability distribution on $\Sigma$ to give
\begin{align*}
4\! \exv{S, S'}{\sup_{h\in \H} \Id{k_{S'}(h) > m'\epsilon(k_S(h))}}
    \!= 4 \exv{\sigma}{\exv{T}{\max_{h\in \H|_T} \Id{ k_{T}(h) > m'\epsilon(k_{T(\sigma)}(h)) + k_{T(\sigma)}(h)}}}\!.
\end{align*}

Next, we decompose the (disjoint) events over the values that $k_{T(\sigma)}(h)$ can take, taking into account that the term $k=m$ of the sum vanishes by definition of $\epsilon(m)$, to obtain
\begin{align}
&4\! \exv{S, S'}{\sup_{h\in \H} \Id{k_{S'}(h) > m'\epsilon(k_S(h))}}\nonumber\\
    &\hspace{50pt}= 4 \exv{\sigma}{\exv{T}{\max_{h\in \H|_T} \sum_{k=0}^{m-1}\Id{ k_{T}(h) > m'\epsilon(k) + k} \Id{k_{T(\sigma)}(h) = k} }}\nonumber\\
    &\hspace{50pt}\le 4 \exv{\sigma}{\exv{T}{\sum_{h\in \H|_T} \sum_{k=0}^{m-1}\Id{ k_{T}(h) > m'\epsilon(k) + k} \Id{k_{T(\sigma)}(h) = k} }}\nonumber\\
    &\hspace{50pt}= 4 \exv{T}{\sum_{h\in \H|_T} \sum_{k=0}^{m-1}\Id{ k_{T}(h) > m'\epsilon(k) + k} \exv{\sigma}{\Id{k_{T(\sigma)}(h) = k} }}.\label{eq:main_3}
\end{align}

Notice that $\exv{\sigma}{\Id{k_{T(\sigma)}(h) = k} }$ corresponds to the probability of choosing $m$ examples (those of $T(\sigma)$), $k$ of which are incorrectly classified by $h$, among the $m + m'$ examples of $T$ (which contains $k_T(h)$ examples incorrectly classified by $h$).
This is a hypergeometric experiment and thus $\exv{\sigma}{\Id{k_{T(\sigma)}(h) = k}} = \hyp(k, m, k_T(h), m+m')$.
Moreover, note that $m'\epsilon(k) + k$ is increasing with $k$ because $\HypInv(k,m,\delta,M)$ is also increasing with $k$, as shown in Lemma~\ref{lem:monotonicity_hyp_tail_inv} of Appendix~\ref{app:hyp_tail_inv}.
Hence, the indicator function $\Id{ k_{T}(h) > m'\epsilon(k) + k}$ simply truncates the end of the summation.
Letting $\kappa(h) \eqdef \max\cb{k \in [m-1] : k_T(h) > m'\epsilon(k) + k}$ and substituting all results back into Inequality~\eqref{eq:main_2}, we have
\begin{align}
\nonumber\prob{S}{\exists h \in \H, R_\D(h) > \epsilon(k_S(h))}
    &< 4 \exv{T}{\sum_{h\in \H|_T} \sum_{k=0}^{\kappa(h)}\hyp(k, m, k_T(h), m+m') }\\
    &=4 \exv{T}{\sum_{h\in \H|_T} \Hyp(\kappa(h), m, k_T(h), m+m') }.\label{eq:original}
\end{align}

On the other hand, because $k_T(h)$ and $\HypInv(k,m,\delta,M)$ must both be integers, we have that $k_T(h) > m'\epsilon(k)+k$ implies $k_T(h) \ge \HypInv(k, m, \frac{\delta}{4\tau_\H(m+m')}, m+m')$.
Furthermore, since $\Hyp(k, m, K, M)$ is decreasing with respect to parameter $K$ as shown in Lemma~\ref{lem:monotonicity_hyp_tail} of Appendix~\ref{app:hyp_tail}, it follows that
\begin{align}
    \Hyp(\kappa, m, k_T(h), m+m') &\le \Hyp(\kappa, m, \HypInv(\kappa, m, \textstyle\frac{\delta}{4\tau_\H(m+m')}, m+m'), m+m')\label{eq:main_4}\\
    &\le \frac{\delta}{4\tau_\H(m+m')},\nonumber
\end{align}
by the definition of $\HypInv$.
Then Inequality~\eqref{eq:original} simplifies to
\begin{equation}\label{eq:main_5}
\prob{S}{\exists h \in \H, R_\D(h) > \epsilon(k_S(h))}
    <4 \exv{T}{\sum_{h \in\H|_T} \frac{\delta}{4\tau_\H(m+m')} }
    \le \delta \frac{\max_{T} \bigl| \H|_T \bigr| }{\tau_\H(m+m')}  =\delta,
\end{equation}
as desired.
\end{proof}

\subsection{The Tightness of Theorem~\ref{thm:main}}
\label{ssec:tightness_thm_main}

The above proof improves upon the derivation of \cite{vapnik98} and \cite{as-93} in three ways.
First, in the ghost sample trick, both \cite{vapnik98} and \cite{as-93} set $m'$ equal to $m$.
\citeauthor{vapnik98} does so arbitrarily, while \citeauthor{as-93} require it to be able to use Hoeffding's inequality.
However, as discussed in the next section, one can improve the bound significantly by making a better choice for $m'$.
Note that \citet{catoni2004improved}'s approach does not use a ghost sample in the same way as what is done here.
However, it requires a ``shadow sample'' which plays a similar role, and they also find that optimizing the shadow sample size is beneficial.

Second, we use the indicator function to truncate the summation on $k$ in Equation~\eqref{eq:main_3}.
This step is neglected by \cite{vapnik98}, and allows us to avoid summing unnecessary terms.
Furthermore, it enables us to proceed to a third improvement: the use of the hypergeometric tail pseudo-inverse, which 
provides the tightest possible upper bound on the right-hand side of Equation~\eqref{eq:original}.

The new bound is nearly optimal in the context of VC classes since few inequalities remain that could be improved.
Let us inspect the aforementioned inequalities.

The first inequality is introduced for the ghost sample trick.
It relies on the result~\eqref{eq:ghost_sample} of \cite{greenberg14} which bounds the probability of a binomial variable exceeding its expectation.
Tighter bounds with the same purpose exist, but they are distribution dependent \citep{pelekis2016lower,doerr2018elementary}.
As \cite{cortes2019relative} observed, the bound of \cite{greenberg14} presents a trade-off between the minimum value $\epsilon$ can take and the factor $4$ in the bound (for example, one could require $\epsilon \ge 2/m'$ instead of $1/m'$ to replace $4\tau_\H$ in the bound by $3.375\tau_\H$).
However optimizing this trade-off appears negligible, since the best we could achieve is to replace the factor $4$ by a factor $2$, and this factor contributes only logarithmically to the bound.

The second inequality is obtained by changing the order of the supremum over the classifiers and the expectation over the ghost sample in Equation~\eqref{eq:main_2}.
Then, in order to use the union bound, a third inequality occurs where we replace the maximum over the classifiers by a sum.
The last inequality occurs at Equation~\eqref{eq:main_5}, where we replace the expectation by a maximum.
These costly steps, which introduce the growth function, are required to be able to use the hypergeometric tail inverse.
To do better, we would need to consider distribution-dependent (possibly localized) complexity measures, such as in \citet{oneto2016local,zhivotovskiy18localization}.
Note that the use of the growth function here is not asymptotically optimal, as it introduces an ``unnecessary'' $\log m$ factor in the bound.
However, typical chaining techniques used to avoid the growth function usually lead to very large constant factors, making the final bound worse for most practical applications, as discussed in Appendix~\ref{app:comparison_lugosi}.

Note that the inequalities used in Equation~\eqref{eq:main_4} are not approximations, but rather consequences of the discrete nature of the hypergeometric distribution.
They are therefore unavoidable and minimal, and tend to disappear as the sample grows in size.

Finally, notice that we can recover known bounds by making approximations.
Indeed, one obtains \citeauthor{vapnik98}'s pessimistic bound when applying Hoeffding's inequality on the hypergeometric tail, which yields $\Hyp(k,m,K,M) \le e^{-2t^2m}$ for $t \eqdef \frac{k}{m} - \frac{K}{M}$ \citep{chvatal1979tail}.
On the other hand, by definition of the pseudo-inverse, we have, for any $\delta$, that $\Hyp\big(k,m,K,M\big) > \delta$ whenever $K \le \HypInv(k,m,\delta,M)-1$.
Combining both results implies $\HypInv(k,m,\delta,M) - 1 < M\sqrt{\frac{-\ln\delta}{2m}} + \frac{kM}{m}$.
Substituting this inequality in the expression of $\epsilon(k)$ of Theorem~\ref{thm:main} (neglecting the first part of the maximum), one obtains
\begin{equation*}\label{eq:bound_hoeffding_reduction}
    \epsilon(k) < \frac{m+m'}{m'}\sqrt{\frac{-\ln\delta+\ln(4\tau_\H(m+m'))}{2m}} + \frac{k}{m}\, .
\end{equation*}
Setting $m'=m$, we end up with the pessimistic bound, but with an extra factor of $\sqrt{2}$ in front of the square root, and without the $1/m$ additive term.
Note that in \citeauthor{vapnik98}'s derivation, this $1/m$ term was originally introduced to satisfy the requirements of the ghost sample trick, whereas our approach handles it by using a maximum in the expression of $\epsilon$.

Furthermore, by inspecting the case where $k=0$ (\ie, the realizable case), we can also simplify the bound to recover the $\widetilde{\Theta}(d/m)$ asymptotic behavior.
It is easy to verify that $\Hyp(0, m, K, M) \le \pr{\frac{M-K}{M}}^m$, which then implies that $\HypInv(0,m,\delta,M) \le M(1-\delta^{1/m})+1$.
Substituting this approximation into Theorem~\ref{thm:main} and then using $1-e^x \le -x$ leads to
\begin{equation*}
    \epsilon(0) \le \pr{1 + \frac{m}{m'}} \frac{1}{m} \ln\pr{ \frac{4\tau_\H(m+m')}{\delta}},
\end{equation*}
which shows that the new bound, similarly to Vapnik's relative deviation bound, correctly interpolates between the realizable and agnostic cases.


\subsection{The Ghost Sample Size Trade-off}
\label{ssec:ghost_sample_size_trade-off}

The upper bound $\epsilon$ of Theorem~\ref{thm:main} presents an interesting dependence in $m'$.
Indeed, increasing $m'$ makes the factor of $\frac{1}{m'}$ shrink, while $\HypInv(k,m,\delta,M)$ decreases when $M$ grows (keeping the other three parameters fixed), and increases when $\delta$ becomes smaller, as shown in Appendix~\ref{app:hyp_tail_inv}.
Since the confidence parameter $\delta$ is divided by the growth function $\tau_\H(m+m')$, which we assume is polynomial in $m+m'$ (for $m+m'$ large enough), we have that at some point, increasing $m'$ becomes less beneficial than decreasing it, thus presenting a non trivial trade-off.

Therefore, we can further tighten the new bound by carefully choosing the optimal value of $m'$.
Due to the discrete nature of the hypergeometric tail pseudo-inverse, we must resort to using a numerical solver to do so.
By finding the best ghost sample size $m'_\text{best}$ for a variety of parameters, we can notice that $m'_\text{best}$ is increasing with $k$ and $m$, and is decreasing with $\delta$ and the VC dimension $d$ of the hypothesis class $\H$.
We observe that for low empirical risk ($k/m \lesssim 10\%$), $m'_\text{best}$ is about four times greater than $m$ (which can provide a good heuristic for finding $m'_\text{best}$). But for some particular sets of parameters $k$, $m$, $\delta$ and $d$, $m'_\text{best}$ can be almost as large as $100\, m$.
This shows that the choice of \citeauthor{vapnik98} and of \citeauthor{as-93} of setting $m'$ equal to $m$ can be very far from the optimal solution.
Computing the new bound for 880 different combinations of parameters, we gather that optimizing the bound leads to a relative gain of 8\% to 10\% over the non-optimized version where $m'=m$, a substantial improvement. 

A more detailed discussion of our findings is presented in Appendix~\ref{app:ghost_sample_size_tradeoff}.
As a final note, we emphasize that Theorem~\ref{thm:main} requires $m'$ to be chosen ahead of time; hence, we cannot optimize its value using the number of errors $k$ made by our classifier.
We can, however, optimize it for an \emph{anticipated} number of errors as some sort of prior knowledge.

\section{More Results Using the Hypergeometric Tail Inversion}
\label{sec:other_results}

The approach we have exploited for deriving the new bound can be adapted to various other settings in a similar fashion, which shows the versatility of the technique.
We consider a relative deviation version of the new bound, a margin bound inspired by the work of \citet{anthonybartlett99}, as well as a lower bound.

\subsection{The Relative Deviation Bound Revisited}
\label{ssec:relative_deviation_bound_revisited}

It is known that relative deviation bounds, which bound $(R_\D(h) - R_S(h))/\sqrt{R_\D(h)}$, often provide a better asymptotic rate than the pessimistic bound when the empirical risk is small.
Bounding the relative deviation with the new technique, we obtain Theorem~\ref{thm:relative_deviation}, presented and proved in Appendix~\ref{app:hyptailinv_reldev}.
From this theorem directly follows the corollary below.

\begin{corollary}
\label{coro:relative_deviation}
Let $\H \subseteq \mathcal{Y}^\mathcal{X}$ be a hypothesis class.
Then for any distribution $\D$ on $\mathcal{X} \times \mathcal{Y}$, any confidence parameter $\delta \in (0, 1)$ and any integers $m, m'>0$, the following inequality holds:
\begin{equation*}
\prob{S\sim\D^m}{\forall h \in \H, R_\D(h) \leq R_S(h) + \frac{\eta(m R_S(h))^2}{2} \pr{ 1 + \sqrt{1 + \frac{4 R_S(h)}{\eta(m R_S(h))^2 }} }} > 1 - \delta,
\end{equation*}
where $\eta(k) = 0$ for $k=m$, and otherwise
\begin{equation*}
\eta(k) = 
    \max\cb{ \scalebox{.9}{
        $\displaystyle \frac{1}{\sqrt{m'}}, \frac{m+m'}{m'} \frac{u(k) - \frac{k}{m}}{\sqrt{u(k)}}$
    }}
    \text{ with }
    u(k) = \frac{\HypInv\big(k, m, \textstyle\frac{\delta}{4 \tau_\H(m + m')}, m + m'\big) - 1}{m+m'}.
\end{equation*}
\end{corollary}

The steps of the proof of Corollary~\ref{coro:relative_deviation} follow closely those of Theorem~\ref{thm:main}---the main exception being the ghost sample trick.
Indeed, to bound the relative deviation, we need to use a looser version of the result of \cite{greenberg14}.
As such, this latest bound can only be looser, a phenomenon observed in Appendix~\ref{app:numerical_comparison} (although the difference between both bounds is small).
As a final note, Appendix~\ref{app:ghost_sample_tradeoff_hti-rd} offers a brief analysis of the ghost sample trade-off for this latter bound, which shows that the benefit of optimizing $m'$ is even more important in this case.

\subsection{A Bound for Multi-class Margin Classifiers}
\label{ssec:margin_classifiers}

Taking inspiration in the work of \citet{anthonybartlett99}, we can also apply our framework to margin classifiers.
In margin theory, we consider hypotheses of the form $h_\f(\x) = \argmax_{i\in[n]} f_i(\x)$ (with arbitrary rule for breaking ties), where $\f \in \F \subseteq ([0,1]^n)^\X$ is a vector-valued function that outputs a score between $0$ and $1$ for each class.
The \emph{margin}
of a classifier $h_\f$ is $\mu(\f(\x), y) \eqdef f_y(\x) - \max_{j\neq y} f_j(\x)$ and can be interpreted as the confidence of $h_\f$ in its prediction.


Contrary to the previous settings, where the hypothesis class $\H|_S$ restricted to a sample $S$ is finite, here the class $\F|_S$ is generally infinite.
This issue can be fixed by considering a discretization of $\F$ instead.
Such a process is called a \emph{cover} of $\F$, from which we can define uniform covering numbers (with respect to a mixed $L_p$-$L_\infty$ metric $\ell$ on $\reals^{n \times m}$) as follows.
\begin{definition}[Uniform covering numbers]
Let $W \subseteq \reals^{n \times m}$ and $c > 0$.
Then, we say that $\widehat{W} \subseteq W$ is a $c$-cover of $W$ if and only if, for every $w \in W$, there exists $\widehat{w} \in \widehat{W}$ such that $\ell(w,\widehat{w}) < c$.
$W$ is totally bounded if, for each $c > 0$, there exists a finite $c$-cover of $W$.
In this case, we define the $c$-covering number of $W$, $\N^p(c, W)$, to be the minimum cardinality of a $c$-cover of $W$ w.r.t. a $L_p$-$L_\infty$ metric. 
Then, for any integer $m>0$ and any $p > 0$, the uniform $c$-covering number of $\F$ is
\begin{equation*}
    \N^p_\F(c, m) =\! \max_{S:\abs{S} =m} \N^p(c, \F|_S),
\end{equation*}
where $\F|_S\eqdef \{ (\f(\x_1),\dots,\f(\x_m)) : \f \in \F \}$ denotes the restriction of $\F$ to the sample $S$.
\end{definition}
These uniform covering numbers are analogous to the growth function and play a similar role.
They allow us to prove the following theorem, as is done in Appendix~\ref{app:margin_bound}.

\begin{theorem}\label{thm:margin_bound}
Let $\F \subseteq ([0,1]^n)^\X$ be a vector-valued function class, $h_\f(\x) = \argmax_{i\in[n]} f_i(\x)$ be a classifier and $R^\gamma_S(\f) \eqdef \frac{1}{m} \sum_{(\x,y)\in S} \Id{\mu(\f(\x)),y) < \gamma}$ be the empirical margin risk.
Then for any distribution $\D$ on $\X \times \Y$, any confidence parameter $\delta \in (0, 1)$, any margin $\gamma \in (0,1)$, any $p\ge1$ and any integers $m$ and $m' > 0$,
\begin{equation*}
\prob{S\sim \D^m}{ \forall \f \in \F, R_\D(h_\f) \leq \epsilon(m R^\gamma_S(\f), m, \delta) } > 1 - \delta,
\end{equation*}
where $\epsilon(k,m,\delta) = 1$ if $k=m$ and otherwise
\begin{equation*}
\epsilon(k, m, \delta) = \frac{1}{m'} \max\cb{ 1, \HypInv\Bigl(k, \frac{\delta}{4 \N^p_\F\pr{\frac{\gamma}{2^{2 - 1/p}}, m + m'}}, m, m + m'\Bigr) - 1 - k }.
\end{equation*}
\end{theorem}

\subsection{A Lower Bound}
\label{ssec:lower_bound}

To derive Theorem~\ref{thm:main}, we relaxed the requirement to bound the absolute value of the difference between the true and empirical risks of a classifier in order to gain more flexibility.
This lead us to consider only an upper bound on the true risk; however, one can use similar arguments to obtain a lower bound that holds simultaneously for all classifiers of a hypothesis class.
This section presents such a bound.

The main observation to be made is that the definition of the pseudo-inverse $\HypInv$ of the hypergeometric tail is no longer appropriate for a lower bound.
As we have already discussed below Definition~\ref{def:hyp_tail_inv}, the definition of a pseudo-inverse is not unique.
In fact, for the lower bound, the correct definition of the pseudo-inverse is the one used by \cite{le2016validation}, which is as follows.
\begin{definition}[Hypergeometric tail lower pseudo-inverse]\label{def:hyp_tail_lower_inv}
We define the lower pseudo-inverse of the hypergeometric tail function, for any $0 < k \le m$ and any $\delta \in (0,1)$ as
\begin{equation*}
    \HypInvLower(k, m, \delta, M) \eqdef \max \cb{ K : \Hyp(k, m, K, M) \ge \delta }.
\end{equation*}
\end{definition}
Note that $\HypInv$ and $\HypInvLower$ differ at most by one.
It is straightforward to adapt the proof of the monotonic properties of $\HypInv$ of Lemma~\ref{lem:monotonicity_hyp_tail_inv} to the lower pseudo-inverse.
Furthermore, while we had $\Hyp(k,m,\HypInv(k,m,\delta,M),M) \le \delta$, the new definition implies instead $\Hyp(k,m,\HypInvLower(k,m,\delta,M),M) \ge \delta$, which is more relevant for a lower bound.

Using the lower pseudo-inverse, our lower bound theorem is stated below.
\begin{theorem}[Hypergeometric tail inversion lower bound]
\label{thm:lower_bound}
Let $\H \subseteq \mathcal{Y}^\mathcal{X}$ be a hypothesis class.
Then for any distribution $\D$ on $\mathcal{X} \times \mathcal{Y}$, any confidence parameter $\delta \in (0, 1)$ and any integers $m, m'>0$, the following inequality holds:
\begin{equation*}
\prob{S\sim\D^m}{\forall h \in \H, R_\D(h) \geq \epsilon(m R_S(h), m, \delta)} > 1 - \delta,
\end{equation*}
where $\epsilon(k,m,\delta) = 0$ if $k=0$ and otherwise
\begin{equation*}
\epsilon(k, m, \delta) = \frac{1}{m'} \min\cb{ m'-1,\; \HypInvLower\big(k-1, m, 1- \textstyle\frac{\delta}{4 \tau_\H(m + m')}, m + m'\big) + 1 - k }.
\end{equation*}
\end{theorem}
The proof is presented in Appendix~\ref{app:proof_lower_bound}, and follows closely the steps of the proof of the upper bound.
Appendix~\ref{app:comp_upper_lower} compares briefly the upper and lower bounds, where we can see that both are perfectly symmetric.

\section{Numerical Comparison to Other Bounds}
\label{sec:numerical_comparison}

To show that the new bound is indeed tighter than previous existing VC bounds, we proceed with some numerical comparisons.
The bounds we consider are \citeauthor{vapnik98}'s pessimistic (VP) bound, given by
\begin{equation*}
    \epsilon_\textnormal{VP}(R_S(h)) = R_S(h) + \frac{1}{m} + \sqrt{\E(m)},
\end{equation*}
where $\E(m) = \frac{1}{m}( \ln 4\tau_\H(2m) - \ln \delta )$, \citeauthor{vapnik98}'s relative deviation (VRD) bound, given by
\begin{equation*}
    \epsilon_\textnormal{VRD}(R_S(h)) = R_S(h) + 2\E(m)\big(1 + \sqrt{1 + R_S(h)/\E(m)} \big),
\end{equation*}
the new hypergeometric tail inversion (HTI) bound (Theorem~\ref{thm:main}), \citet{lugosi2002pattern}'s chaining bound (Theorem 1.16) given by
\begin{equation*}
    \epsilon_\textnormal{Lugosi}(R_S(h)) = R_S(h) + 24\sqrt{\frac{2d}{m}} \pr{ \sqrt{a} + \frac{\sqrt{\pi} e^a}{2} (1- \text{erf}(\sqrt{a})) } + \sqrt{\frac{-\ln \delta}{2m}},
\end{equation*}
with
\begin{equation*}
    a \eqdef \frac{(d+1)(2+\ln2)}{2d},
\end{equation*}
and finally \citet{catoni2004improved}'s Theorem~4.6 (C4.6), given by
\begin{equation*}
    \epsilon_{C4.6}(R_S(h)) = \pr{1 + \frac{2d'}{m}}^{-1} \pr{ R_S(h) + \frac{d'}{m} + \frac{1}{m}\sqrt{2d'mR_S(h)(1-R_S(h)) + d'^2} },
\end{equation*}
with $d' = \big(\frac{m+m'}{m'}\big)^2\pr{\ln \tau_\H(m+m') - \ln \delta}$.
This bound, which is the tightest among those present in the work of \citet{catoni2004improved}, is only valid when both $R_S(h)$ and $\epsilon_{C4.6}(R_S(h))$ are less than or equal to $0.5$.
Here, $m'$ is the ``shadow'' sample size of Catoni.
In Appendix~\ref{ssec:sc_comparaison}, we also compare against sample compression bounds~\citep{floyd95sample}, which can, in some situations, be used instead of VC bounds.

We assume a binary classification problem so that we can parametrize the growth function in terms of the VC dimension via Sauer-Shelah's lemma.
We would like our comparison to be representative of real-life applications; however, the VC dimension depends greatly on the hypothesis class chosen for the problem at hand.
For example, a VC dimension of $50$ is probably large for polynomial regression, approximately correct for medium sized decision trees \citep{leboeuf2020decision}, and small for neural networks \citep{bartlett2019nearly}.
Figure~\ref{fig:bounds_comp} presents the different bounds on the true risk as functions of the empirical risk for a moderate VC dimension, and as a function of the VC dimension for a small empirical risk.
We set the confidence parameter $\delta$ to $5\%$.
We showcase more sets of parameters in Appendix~\ref{ssec:other_numerical_settings}, where we can see that this choice of parameters is fair for all four bounds.
The code to produce all figures is also available in the GitHub repository\footnote{\url{https://github.com/jsleb333/hypergeometric_tail_inversion/tree/main/scripts}}.

\begin{figure}[h!]
\centering
\begin{subfigure}[t]{0.485\textwidth}
    \centering
    \includegraphics[width=\textwidth]{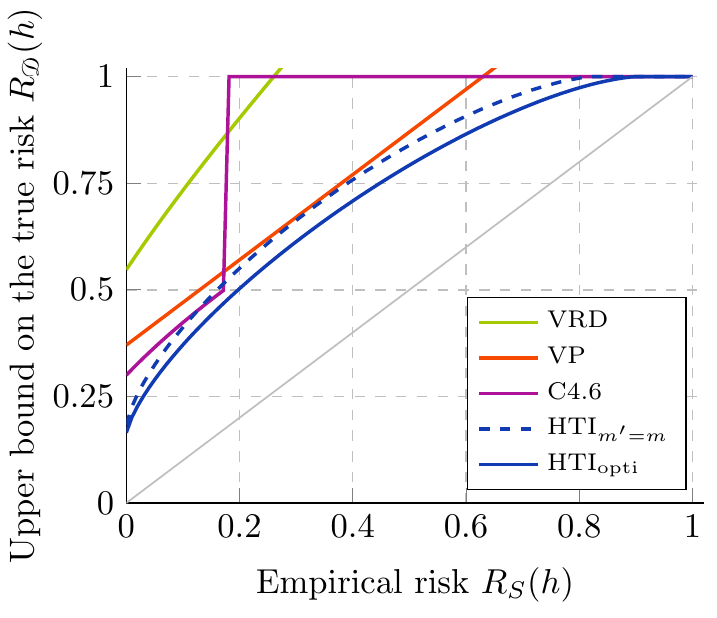}
    \caption{Bounds as functions of the empirical risk. The gray line corresponds to the empirical risk. The dashed line is the non-optimized HTI bound.}
    \label{fig:bounds_comp_risk}
\end{subfigure}\hfill
\begin{subfigure}[t]{0.485\textwidth}
    \centering
    \includegraphics[width=\textwidth]{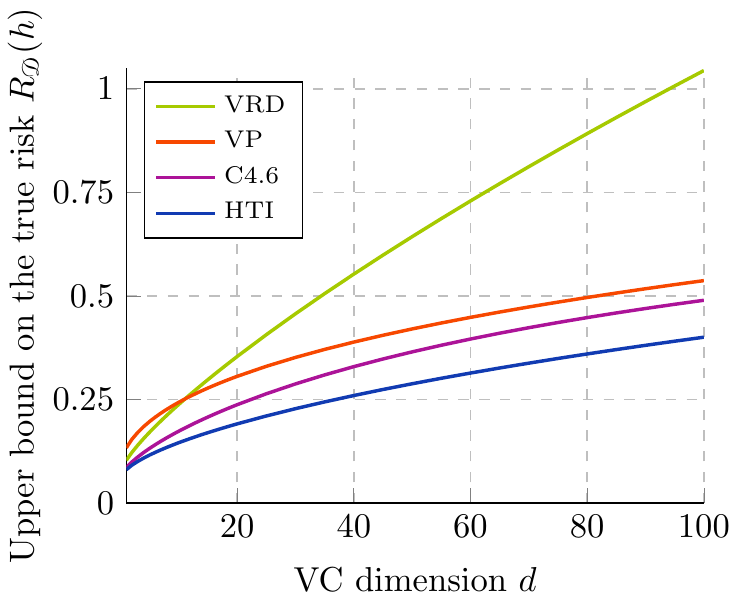}
    \caption{Bounds as functions of the VC dimension.}
    \label{fig:bounds_comp_d}
\end{subfigure}
\caption{
Bounds for a sample size of $m=2000$ and a confidence parameter $\delta=0.05$.
Left: the bounds as functions of the empirical risk and a VC dimension of $d=50$. Right: the bounds as functions of the VC dimension for $R_S(h) = 0.05$.
Lugosi's chaining bound does not appear as it is vacuous for all values of $R_S(h)$ and $d$ considered.
}
\label{fig:bounds_comp}
\end{figure}

In Figure~\ref{fig:bounds_comp_risk}, the HTI\textsubscript{opti} bound (solid blue) presents the bound with $m'=6970$, which is optimized for $k=0$, while the HTI$_{m'=m}$ bound (dashed blue) presents the bound for $m' = m = 2000$.
One can appreciate the improvement provided by the optimization, which is non-negligible for almost all risks.
The shadow sample size of Catoni's C4.6 (purple) bound is also optimized for 0 error, but in this case $m'$ can only be an integer multiple of $m$.
We find $m'=16m$ to be the best choice.
Note that Lugosi's chaining bound is vacuous for these sets of parameters and therefore does not show in the figure.
The new HTI bound is always tighter than the other bounds, and is significantly more for small empirical risks.
Indeed, at $R_S(h)=0$, the HTI bound is equal to $0.165$, the C4.6 bound to $0.300$, the VP (orange) bound to $0.370$, and the VRD (green) bound to $0.547$.
Hence, for these parameters, the HTI bound is almost twice as tight as Catoni's bound.
Notice that the HTI bound becomes vacuous only for empirical risks over approximately $0.85$, contrary to Catoni's bound, which is only valid when both the $R_S(h)$ and the bound are lower than $0.5$.
Since the worst risk a classifier can realize is given by random guessing, this constraint on Catoni's bound is inconsequential for binary classification.
However, for $n$ classes problems, random guessing yields a risk of $1-\frac{1}{n}$, which can be much worse than $0.5$.

In Figure~\ref{fig:bounds_comp_d}, the ghost samples size of the HTI bound and of Catoni's bound are optimized for each value of $d$ for an empirical risk of $0.05$.
Again, the new HTI bound supersedes all previous bounds for all values of VC dimension (Lugosi's bound does not show because it is vacuous).
Furthermore, even though the bounds are similar for small VC dimension, the relative gain of the HTI bound against the others increases as $d$ increases, and for $d=100$, the HTI bound is more than $1.5$ times tighter than Catoni's bound.

To compare the asymptotic convergence rates of the different bounds, Figure~\ref{fig:bounds_comparison_m} presents, in a log-scale graph, the bounds on $R_\D(h) - R_S(h)$ as functions of the sample size ($m=50$ up to $1$,$000$,$000$), for a zero empirical risk $R_S(h) = 0$ (realizable case) and non-zero empirical risk of $R_S(h) = 0.1$ (agnostic case).
In both cases, the VC dimension is set to a moderate value of $d=50$ and a confidence parameter of $\delta=0.05$.
The ghost sample size of the HTI and C4.6 bounds are optimized independently for each value of $m$ and for an empirical risk of $0$ (left) and $0.1$ (right).
The optimal convergence rates ($\Theta(\frac{d}{m}\log \frac{m}{d})$ in the realizable case and $\Theta(\sqrt{d/m})$ in the agnostic case) are plotted in gray for comparison purposes.

\begin{figure}[h]
\centering
\begin{subfigure}[t]{0.485\textwidth}
    \centering
    \includegraphics[width=\textwidth]{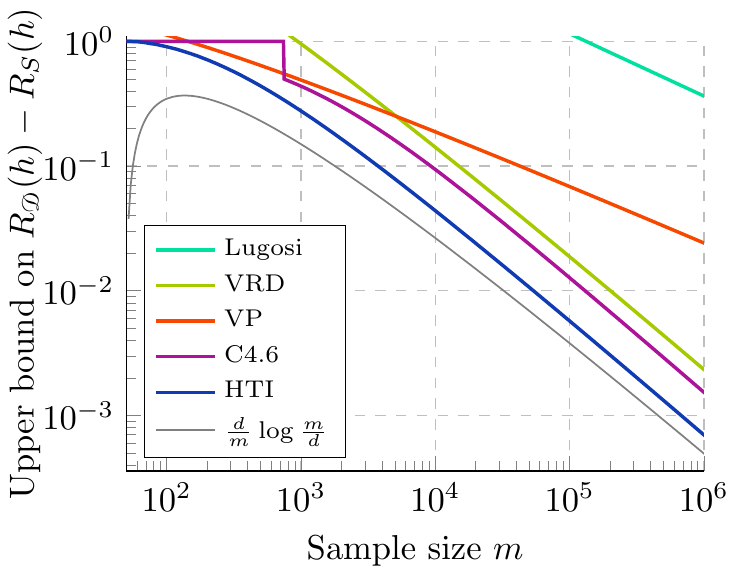}
    \caption{Bounds as functions of the sample size for $R_S(h)=0$.}
    \label{fig:bounds_comp_m_risk=0_d=50}
\end{subfigure}\hfill
\begin{subfigure}[t]{0.485\textwidth}
    \centering
    \includegraphics[width=\textwidth]{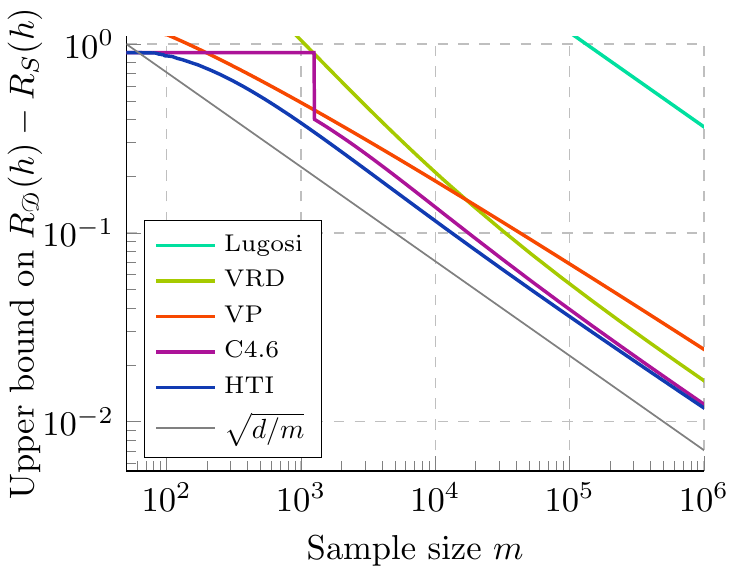}
    \caption{Bounds as functions of the sample size in for $R_S(h)=0.1$.}
    \label{fig:bounds_comp_m_risk=0.1_d=50}
\end{subfigure}
\caption{Comparison of the bounds on $R_\D(h) - R_S(h)$ as functions of the sample size $m$ for a VC dimension $d=50$ and a confidence parameter $\delta=0.05$. Left is the realizable case, right is the agnostic case. The optimal asymptotic rates are shown in gray.
}
\label{fig:bounds_comparison_m}
\end{figure}

In both Figures~\ref{fig:bounds_comp_m_risk=0_d=50} and \ref{fig:bounds_comp_m_risk=0.1_d=50}, we can observe that the new HTI bound is tighter than all the other bounds across the values of $m$ considered.
Moreover, the improvements is significant over Catoni's C4.6 bound for all values of $m$ in the realizable case, and for small sample sizes (up to $m\approx 1250$, where Catoni's bound becomes non-vacuous) for the agnostic case.

In the realizable case, the HTI bound becomes non-vacuous as soon as $m\approx 54$, which is only 4 more than the VC dimension, and it takes a value below $0.5$ at around $m\approx 400$.
In contrast, Catoni's C4.6 bound starts being informative at $m\approx 750$, almost twice the sample size needed for the HTI bound.
The situation is similar for the agnostic case, where the HTI bound is less than one for $m\ge 66$ and is less than one half for $m\gtrsim 565$.
This is less than half of the sample size required ($m\approx 1260$) by Catoni's bound to be non-vacuous.

Remark that in both cases, Lugosi's chaining bound starts becoming non-vacuous only at $m\gtrsim 100$,$000$.
Furthermore, even if it possesses an optimal convergence rate for the agnostic case, recall from the introduction that it would become tighter than Vapnik's pessimistic bound only for $m\gtrsim 10^{335}d$.
Let us also note that, even though the new HTI bound contains an extra unnecessary $\sqrt{\log(m/d)}$ in the agnostic case, it is still more than 30 times smaller than Lugosi's optimal-rate chaining bound at $m\approx 10^6$---a fact which highlights that constant factors matter more than logarithmic ones in practical settings.
As a final observation, the HTI bound has the non-negligible advantage that it has the correct asymptotic rate for the realizable case, something the chaining bound does not, as seen in Figure~\ref{fig:bounds_comp_m_risk=0_d=50}.

\section{Conclusion}

In conclusion, we proposed a new way to improve generalization bounds for VC classes.
Taking inspiration from~\citet{langford05}, we considered the hypergeometric tail pseudo-inverse, from which we derived a novel non-uniform distribution-independent very tight bound with few approximations.
We tightened even more the bound by optimizing the ghost sample size.
Applying the new technique to other settings, we developed a relative deviation bound---which turned out less tight than the novel bound---as well as a bound for multi-class margin classifiers and a lower bound.
We provided numerical comparisons against the previously existing bounds of \citet{vapnik98}, \citet{catoni2004improved} and \citet{lugosi2002pattern}, which has shown that the new generalization bound is significantly tighter for multiple settings, and that it is almost never vacuous.
These experimentations give evidence that constant factors are at the least as important as logarithmic factors in asymptotic rates for practical settings.

Even though the improvements to the constant factor achieved here have proven relevant for practical applications, the hypergeometric tail inversion bound is still not asymptotically optimal in the agnostic case.
Therefore, it would be very interesting to see if chaining techniques could be combined profitably with the hypergeometric tail inversion to obtain even further gain for (reasonably) large sample sizes.



\acks{This work was supported by NSERC Discovery grant RGPIN-2016-05942, by NSERC ES D scholarship PGSD3–505004–2017 and by NSERC BRPC scholarship BRPC-540188-2019.}


\newpage

\appendix

\section{Comparison between Lugosi's chaining bound and Vapnik's pessimistic bound}
\label{app:comparison_lugosi}

In the introduction, we have claimed that Lugosi's chaining bound becomes better than Vapnik's pessimistic bound only for extremely large sample sizes of $m \gtrsim 10^{335}d$.
We show the details of the computation here.

We use \citet{lugosi2002pattern}'s derivations of both bounds for the comparison.
These results are expressed as bounds on the expected deviation of the empirical risk from the true risk, \ie, of the form
\begin{equation*}
    \exv{S}{\sup_{h\in\H} \abs{R_\D(h) - R_S(h)}} \le \beta.
\end{equation*}

These results can be converted to generalization bounds of the form
\begin{equation*}
    \prob{S}{\exists h \in \H : \abs{R_\D(h) - R_S(h)} > \epsilon} \le \delta
\end{equation*}
by using McDiarmid's inequality (also called the bounded difference inequality) (Theorem~1.8 of \citet{lugosi2002pattern}), where $\epsilon \ge \beta + \sqrt{\frac{-\ln \delta}{2m}}$.
Therefore, comparing the bounds on the expected deviation $\beta$ is equivalent to compare the bounds $\epsilon$.

For Vapnik's pessimistic bound, we have $\beta_{\text{pessimistic}} = \sqrt{2\frac{\ln 2 \tau_\H(2m)}{m}}$ (Theorem~1.9 of \citet{lugosi2002pattern} without the last approximation of the proof).
Using Sauer-Shelah's lemma to approximate the growth function by $\tau_\H(2m) \le \big(\frac{2em}{d}\big)^d$, we get
\begin{equation*}
    \beta_{\text{pessimistic}} \le \sqrt{2\frac{\ln(2) + d\ln(2em/d)}{m}} \le \sqrt{\frac{2(d+1)\ln(4em/d)}{m}},
\end{equation*}
where we used $\ln 2 \le \ln (2em/d)$ to simplify things (which will favor the chaining bound in the comparison).

For the chaining bound, we use \citeauthor{lugosi2002pattern}'s Theorem~1.16:
\begin{equation}\label{eq:chaining_1}
    \beta_{\text{chaining}} = \frac{24}{\sqrt{m}} \max_{S:|S|=m} \int_0^1\!\! \sqrt{\ln \pr{ 2 \N(r, \H|_S)}} \textrm{d}r,
\end{equation}
where $\N(r, \H|_S)$ is the $r$-covering number of $\H|_S$.
When the VC dimension of $\H$ is $d$, Haussler showed that
\begin{equation*}
    \N(r, \H|_S) \le e (d+1) \pr{\frac{2e}{r^2}}^d.
\end{equation*}
Substituting this inequality in \eqref{eq:chaining_1} and proceeding to the integration, one has (without further approximations):
\begin{equation*}
    \beta_{\text{chaining}} \le 24\sqrt{\frac{2d}{m}} \pr{ \sqrt{a} + \frac{\sqrt{\pi} e^a}{2} (1- \text{erf}(\sqrt{a})) },
\end{equation*}
with
\begin{equation*}
    a \eqdef \frac{(d+1)(2+\ln2)}{2d}.
\end{equation*}

The term $\frac{\sqrt{\pi} e^a}{2} (1- \text{erf}(\sqrt{a}))$ is always positive, and can be shown to be less than $\frac{\sqrt{\pi}}{2} < 1$.
Neglecting this term (which favors the chaining bound in the comparison) and setting it approximately equal to the pessimistic bound, we have
\begin{align*}
    24 \sqrt{\frac{(d+1)(2+\ln 2)}{m}} \approx \sqrt{\frac{2(d+1)\ln(2em/d)}{m}}.
\end{align*}
Evaluating constants, it simplifies to
\begin{align*}
    775.63 \approx \ln(2em/d).
\end{align*}
Taking the exponential on each side and changing to base 10 gives the claimed result.

\newpage

\section{Overview of the hypergeometric distribution}
\label{app:hypergeometric_distribution}

This Appendix is dedicated to the hypergeometric distribution.
We cover the formal definition of the tail and its pseudo-inverse, their monotonicity properties, and algorithms to compute the pseudo-inverse.

\subsection{The hypergeometric distribution}
\label{app:hyp_tail}

The hypergeometric distribution is an elementary distribution that occurs often.
It deals with the probability of obtaining a given amount of successes out of some number of random draws (without replacement) from a population with a fixed total amount of successes.
This is similar to the binomial distribution, with the exception that the probability of success may change with each draw, since we draw without replacement when sampling according to a hypergeometric distribution.
The following definition gives the mathematical formula for the probability mass function of the hypergeometric distribution.

\begin{definition}[Hypergeometric distribution probability mass function]
Given an $M$-element set containing $K$ successes, the action of drawing $k$ successes in a sample of $m$ elements is called a \emph{hypergeometric experiment} and its probability is given by the \emph{hypergeometric distribution probability mass function}, defined by
\begin{equation*}
    \hyp(k,m,K,M) \eqdef \frac{\binom{K}{k}\binom{M-K}{m-k}}{\binom{M}{m}},
\end{equation*}
where $\binom{N}{n}$ is a binomial coefficient.
\end{definition}

The tail function, also called the cumulative distribution, is particularly useful in our work.
We give its formal definition below.

\begin{definition}[Hypergeometric tail function]
\label{def:app_hyp_tail}
Given a set of $M$ elements containing $K$ successes, the probability of drawing $k$ successes or less in a sample of $m$ elements is defined by the \emph{hypergeometric tail function} (also called the cumulative distribution function) and is expressed as
\begin{equation*}
    \Hyp(k,m,K,M) \eqdef \sum_{j=0}^k\frac{\binom{K}{j}\binom{M-K}{m-j}}{\binom{M}{m}}.
\end{equation*}
\end{definition}

Figure~\ref{fig:hyp_tail_plot} shows the hypergeometric tail function as a function of its four parameters for various sets of default parameters to help the reader picture the behavior of this distribution.

\begin{figure}[p]
\centering
\begin{subfigure}[t]{0.485\textwidth}
    \centering
    \includegraphics[width=\textwidth]{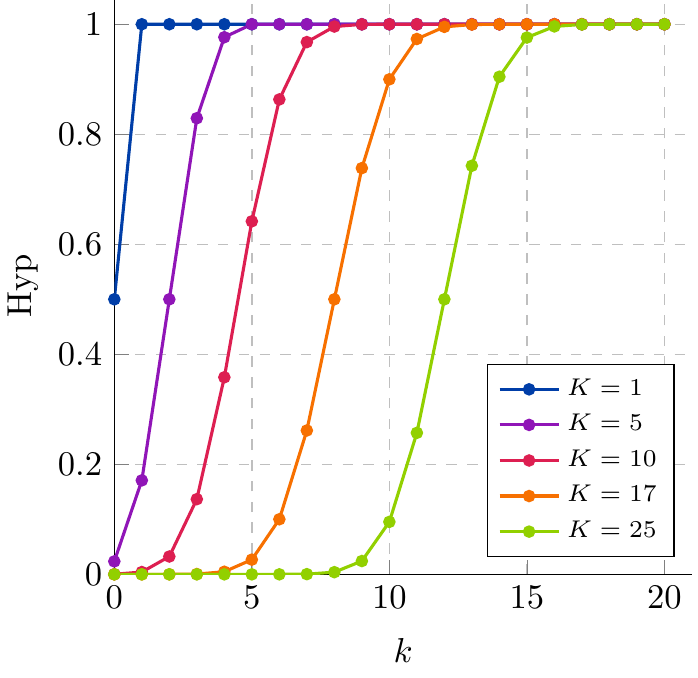}
    \caption{$\Hyp(k,m,K,M)$ as a function of $k$ for various values of $K$.}
    \label{fig:hyp_tail_plot_k}
\end{subfigure}\hfill
\begin{subfigure}[t]{0.485\textwidth}
    \centering
    \includegraphics[width=\textwidth]{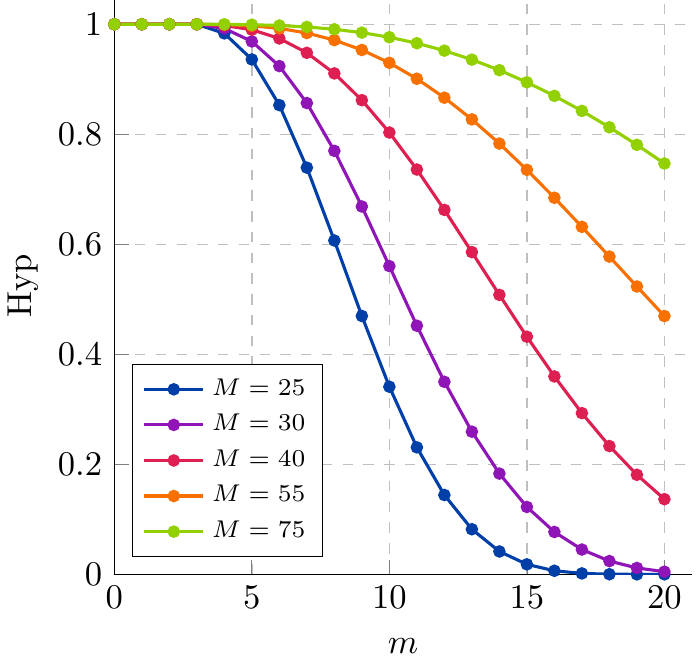}
    \caption{$\Hyp(k,m,K,M)$ as a function of $m$ for various values of $M$.}
    \label{fig:hyp_tail_plot_m}
\end{subfigure}

\vspace{15pt}

\begin{subfigure}[t]{0.485\textwidth}
    \centering
    \includegraphics[width=\textwidth]{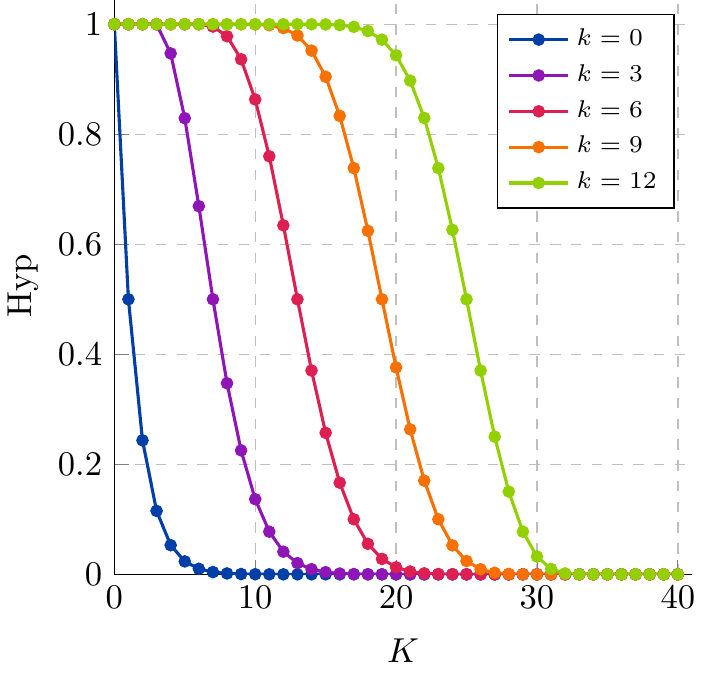}
    \caption{$\Hyp(k,m,K,M)$ as a function of $K$ for various values of $k$.}
    \label{fig:hyp_tail_plot_K}
\end{subfigure}\hfill
\begin{subfigure}[t]{0.485\textwidth}
    \centering
    \includegraphics[width=\textwidth]{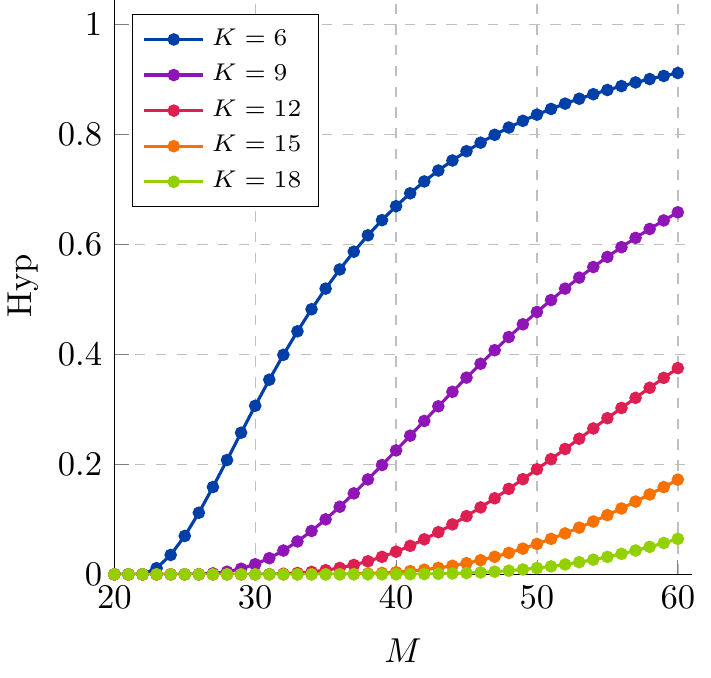}
    \caption{$\Hyp(k,m,K,M)$ as a function of $M$ for various values of $K$.}
    \label{fig:hyp_tail_plot_M}
\end{subfigure}
\caption{Hypergeometric tail function as a function of each of its four parameters $k$, $m$, $K$ and $M$, presented in this order. If not specified, the default parameters are $k=3$, $m=20$, $K=10$ and $M=40$. The hypergeometric tail is discrete and so the lines joining the marks are only there to ease the reading. One can observe the monotonicity properties of the hypergeometric tail function.}
\label{fig:hyp_tail_plot}
\end{figure}

Recently, \cite{berkopec07} has found a new expression for the hypergeometric tail, which is stated in the following proposition.
\begin{proposition}[Berkopec's identity]\label{prop:berkopec}
For any integers $k$, $m$, $K$ and $M$ such that the inequality $0 \le k \le \min\cb{m,K} \le \max\cb{m,K} \le M$ holds, we have that
\begin{equation*}
    \sum_{j=0}^{k} \binom{K}{j} \binom{M-K}{m-j} = \sum_{J=K}^{M-m+k} \binom{J}{k} \binom{M-J-1}{m-k-1},
\end{equation*}
where it is assumed that $\binom{N}{-n} = 0$ for any $N \in \mathds{Z}$ and $n>0$.
\end{proposition}

With this identity, we have 
\begin{equation*}
    \Hyp(k,m,K,M) = \sum_{J=K}^{M-m+k}\frac{\binom{J}{k} \binom{M-J-1}{m-k-1}}{\binom{M}{m}}.
\end{equation*}
This will prove particularly useful when showing the monotonicity properties of the hypergeometric tail, since it allows us to exchange the sum over values of $k$ with a sum over values of $K$.
This brings us to the following lemma.

\begin{lemma}[Monotonicity of the hypergeometric tail function]
\label{lem:monotonicity_hyp_tail}
Let $\Hyp(k,m,K,M)$ be the tail of the hypergeometric distribution as described in Definition~\ref{def:app_hyp_tail}, and let $\binom{N}{-n} \eqdef 0$ for all integers $N \in \mathds{Z}$ and $n > 0$.
Then, for all $k \in \mathds{Z}$ and $0 \le m,K \le M$, $\Hyp(k,m,K,M)$ is increasing in its first and last parameters $k$ and $M$, and decreasing in its second and third parameters $m$ and $K$.
Furthermore, the monotonicity is strict whenever the inequalities $0 \le k < \min(m,K)$ and $m-k \le M-K$ hold simultaneously.
\end{lemma}

One can observe the monotonicity of the tail in Figure~\ref{fig:hyp_tail_plot}.

\begin{proof}
We begin with the first parameter, $k$, fixing $m$, $K$ and $M$.
To show that the tail is monotonically increasing, we need to show that $\Hyp(k+1,m,K,M)-\Hyp(k,m,K,M) \geq 0$ for all $k\in \mathds{Z}$. Starting from Definition~\ref{def:app_hyp_tail}, we have
\begin{align*}
    \Hyp(k,m,K,M)-\Hyp(k-1,m,K,M)
    &= \frac{\binom{K}{k}\binom{M-K}{m-k}}{\binom{M}{m}} \ge 0.
\end{align*}
Moreover, this is strictly greater than zero whenever $0 \le k \le m$ and $K-M+m \le k \le K$, which is the standard support for the distribution.

For the second and third parameters $m$ and $K$, first observe that the tail is symmetric under the exchange of $m$ and $K$, \ie, $\Hyp(k,m,K,M) = \Hyp(k,K,m,M)$.
This can be shown by reorganizing the binomial coefficients of each term in the sum.
This tells us that monotonic decreasing in $m$ implies monotonic decreasing in $K$ and vice-versa.

We will show monotonicity in $K$ with the help of Berkopec's identity \citep{berkopec07} stated in Proposition~\ref{prop:berkopec}.
We have
\begin{align*}
    \Hyp(k,m,K,M)-\Hyp(k,m,K+1,M)
    &= \frac{\binom{K}{k} \binom{M-K-1}{m-k-1}}{\binom{M}{m}} \ge 0,
\end{align*}
and we have strict monotonicity whenever $0 \leq k \leq K$ and $1 \le m-k \le M-K$.
Monotonicity in $m$ follows by exchanging $m$ with $K$ in the previous argument, and is strict when $0\le k \le m$ and $1 \le K-k \le m-K$.

Finally, to show the tail is monotonic increasing in $M$, consider a population of $M+1$ and decompose the probability $\Hyp(k,m,K,M+1)$ into two events: we picked the $(M+1)$-th element among the $m$ elements drawn, or we did not.
The former case has a probability $p \eqdef \frac{m}{M+1}$ of happening and must be multiplied by the probability of drawing $k$ successes among $m-1$ elements within a population of $M$ elements with $K$ successes.
The latter case happens with probability $1-p$ and we multiply it by the probability of drawing $k$ successes among $m$ elements within a population of $M$ elements with $K$ successes.
Putting everything together, we obtain
\begin{align*}
    \Hyp(k,m,K,M+1)
    &= \frac{m}{M+1}\cdot \Hyp(k,m-1,K,M) + \frac{M+1-m}{M+1}\cdot\Hyp(k,m,K,M)\\
    &\ge \frac{m}{M+1}\cdot \Hyp(k,m,K,M) + \frac{M+1-m}{M+1}\cdot\Hyp(k,m,K,M)\\
    &= \Hyp(k,m,K,M),
\end{align*}
where we used the monotonicity of $\Hyp(k,m,K,M)$ in the parameter $m$ at the second step, which implies strict monotonicity in $M$ for the same intervals as $m$.
Alternatively, note that it is possible to prove the first equality by using Pascal's identity inside Definition~\ref{def:app_hyp_tail}.

Taking the intersection of all the intervals of strict monotonicity, we find the intervals specified in the lemma, which concludes the proof.
\end{proof}

\subsection{The pseudo-inverse of the hypergeometric tail}
\label{app:hyp_tail_inv}

Our main result requires us to invert the tail over the parameter $K$.
This differs from the quantile function (also known as the ``percent point'' function), which inverts the tail with respect to the parameter $k$.
This section is concerned with this particular inverse and its subtleties.

Strictly Monotonic functions of a real variable enjoy the property of being invertible.
However, the hypergeometric tail having discrete parameters, defining an inverse is not as straightforward.
Instead, one must consider a pseudo-inverse.
Note that there exist many ways to define such an inverse; the next definition states the pseudo-inverse appropriate for our purpose.

\begin{definition}[Hypergeometric tail pseudo-inverse]\label{def:app_hyp_tail_inv}
The pseudo-inverse of the hypergeometric tail function, for any $k < m$ and any $\delta \in (0,1)$, is defined as
\begin{equation*}
    \HypInv(k, m, \delta, M) \eqdef \min \cb{ K : \Hyp(k, m, K, M) \le \delta }.
\end{equation*}
\end{definition}
We restrict the domain of the pseudo-inverse to $k < m$ because $\Hyp(m,m,K,M) = 1$ for all $K$ and $M$ greater than or equal to $m$, and thus it is ill-defined for $k=m$.
Notice that the domain of $K$ is restricted to $k < K \le M - m + k + 1$ since $\Hyp(k,m,K,M)=1$ for every $K \leq k$, and because we have that $\Hyp(k,m,K,M)=0$ for every $K > M-m+k$.

In particular, this choice of pseudo-inverse implies $\Hyp(k,m,\HypInv(k,m,\delta,M),M) \le \delta$.
Figure~\ref{fig:hyp_tail_plot_inv} gives a visualization of the pseudo-inverse as a function of its four parameters for various sets of default parameters.

\begin{figure}[p]
\centering
\begin{subfigure}[t]{0.47\textwidth}
    \centering
    \includegraphics[width=\textwidth]{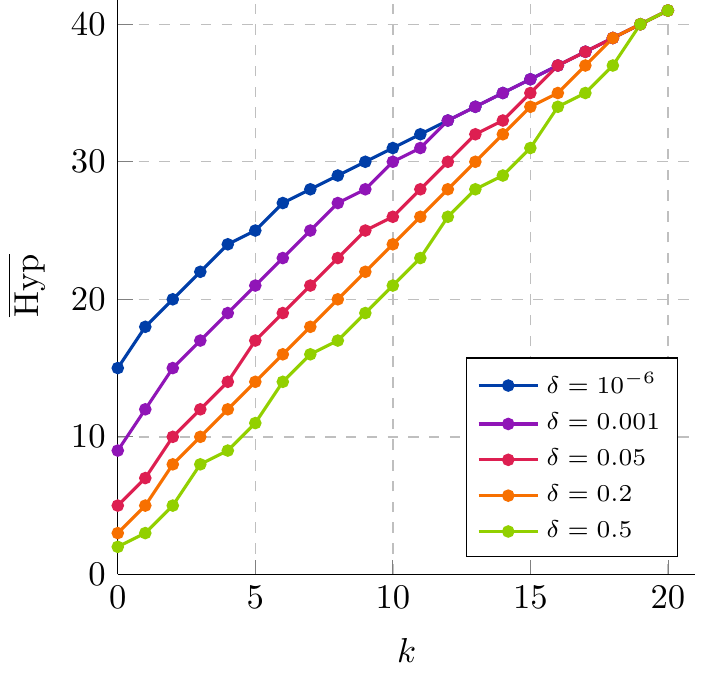}
    \caption{$\HypInv(k,m,\delta,M)$ as a function of $k$ for various values of $\delta$.}
    \label{fig:hyp_tail_inv_plot_k}
\end{subfigure}\hfill
\begin{subfigure}[t]{0.47\textwidth}
    \centering
    \includegraphics[width=\textwidth]{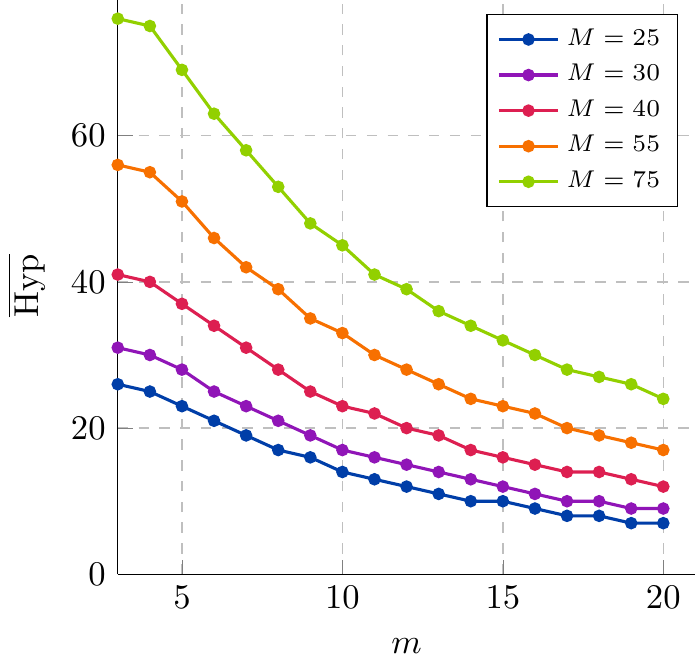}
    \caption{$\HypInv(k,m,\delta,M)$ as a function of $m$ for various values of $M$.}
    \label{fig:hyp_tail_inv_plot_m}
\end{subfigure}

\vspace{15pt}

\begin{subfigure}[t]{0.47\textwidth}
    \centering
    \includegraphics[width=\textwidth]{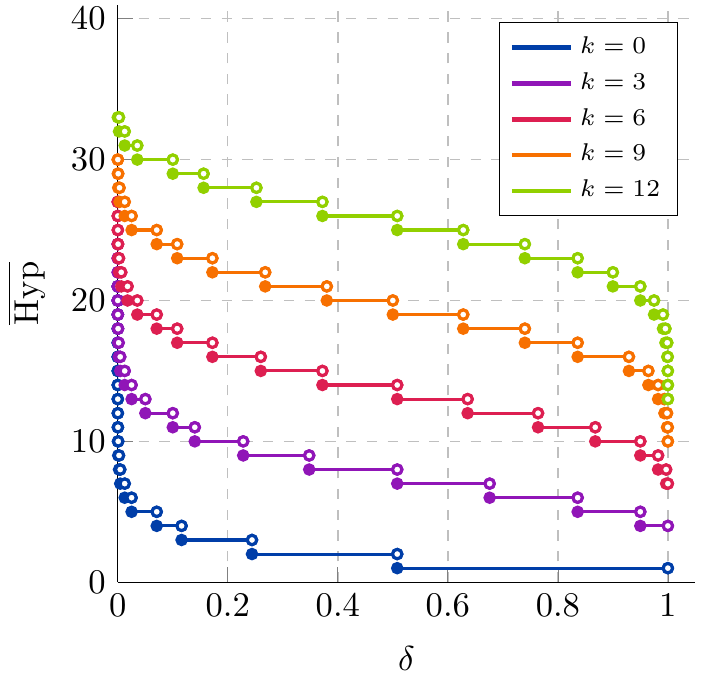}
    \caption{$\HypInv(k,m,\delta,M)$ as a function of $\delta$ for various values of $k$. Contrary to the other parameters, $\HypInv$ is piecewise continuous in $\delta$. Full marks indicate that the endpoint is closed while empty marks mean an open endpoint.}
    \label{fig:hyp_tail_plot_inv_delta}
\end{subfigure}\hfill
\begin{subfigure}[t]{0.47\textwidth}
    \centering
    \includegraphics[width=\textwidth]{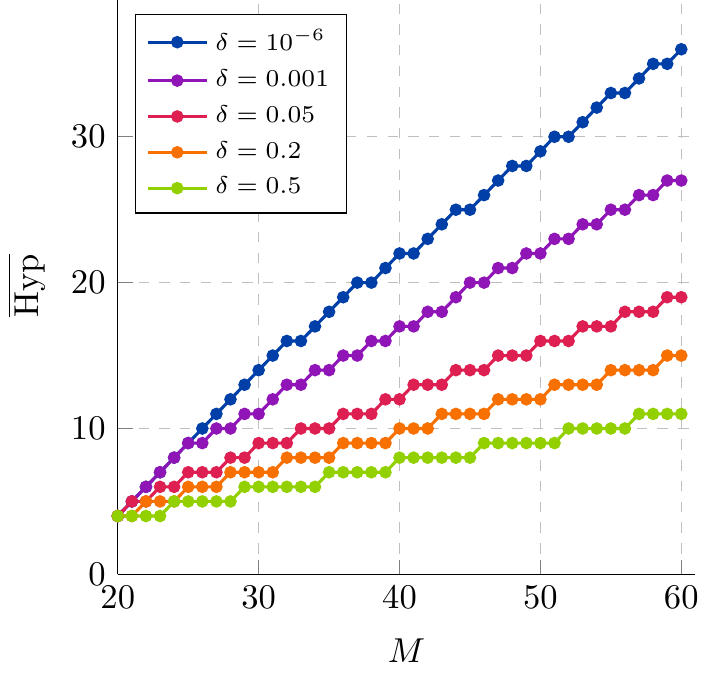}
    \caption{$\HypInv(k,m,\delta,M)$ as a function of $M$ for various values of $\delta$.}
    \label{fig:hyp_tail_plot_inv_M}
\end{subfigure}
\caption{Hypergeometric tail pseudo-inverse as a function of each of its four parameters $k$, $m$, $\delta$ and $M$, presented in this order. If not specified, the default parameters are $k=3$, $m=20$, $\delta=0.05$ and $M=40$. The pseudo-inverse is discrete in $k$, $m$ and $M$, and so the lines joining the marks are only there to ease the reading for these parameters. One can observe the monotonicity properties of the hypergeometric tail pseudo-inverse.}
\label{fig:hyp_tail_plot_inv}
\end{figure}

It is well known that the inverse of a strictly monotone function of a real variable is also strictly monotone.
However, this is less evident for pseudo-inverses.
The following lemma shows that monotonicity properties still hold for the hypergeometric tail pseudo-inverse.

\begin{lemma}[Monotonicity of the hypergeometric tail pseudo-inverse]\label{lem:monotonicity_hyp_tail_inv}

Let $\HypInv(k,m,\delta,M)$ be the pseudo-inverse of the hypergeometric tail function as described in Definition~\ref{def:app_hyp_tail_inv}.
Then, for $\delta\in (0,1)$, for $0 \le k < \min(m,K)$ and for $m-k \le M-K$, $\HypInv(k,m,\delta,M)$ is increasing in its first and last parameters $k$ and $M$, and decreasing in its second and third parameters $m$ and $\delta$.
Furthermore, the monotonicity in $k$ is strict.
\end{lemma}

\begin{proof}
We start by presenting an approach to showing the monotonicity properties of the pseudo-inverse in any of its parameters.
The approach is quite general, but for the sake of clarity, we apply it only for the parameter $k$ and explain how to adapt it for the other parameters.
We then conclude by a more specific approach to show the pseudo-inverse is strictly increasing in $k$.

Let $m$, $\delta$ and $M$ take some values that satisfy the requirements of Lemma~\ref{lem:monotonicity_hyp_tail_inv} and choose $k<m-1$.
Let $\K(k,m,\delta,M) \eqdef \cb{ K : \Hyp(k,m,K,M) \le \delta }$ such that we have $\HypInv(k,m,\delta,M) = \min \K(k,m,\delta,M)$.
We want to show that $\HypInv(k,m,\delta,M)$ is increasing with $k$.
First observe that if, for some $K$, we have that $\Hyp(k+1,m,K,M) \le \delta$, then we know that $\Hyp(k,m,K,M) \le \delta$ is true too by the monotonic increasing nature of the hypergeometric tail in the parameter $k$.
Therefore, for all $K' \in \K(k+1,m,\delta,M)$, we must have $K' \in \K(k,m,\delta,M)$ too.
Hence, we have that $\K(k,m,\delta,M) \supseteq \K(k+1,m,\delta,M)$.
Then, muting dependence on $m$, $\delta$ and $M$ for clarity, and assuming that $\K(k+1)$ is not empty, we have
\begin{align*}
\HypInv(k,m,\delta,M) 
    &= \min \K(k)\\
    &= \min \K(k+1) \cup \pr{ \K(k)\backslash \K(k+1)}\\
    &= \min \cb{ \min \K(k+1), \min \K(k)\backslash \K(k+1) }\\
    &\le \min \K(k+1)\\
    &= \HypInv(k+1,m,\delta,M),
\end{align*}
as desired.
The non-emptiness of $\K(k+1,m,\delta,M)$ requires that $k+1 < m$, because we know $\Hyp(m,m,K,M)=1$, but $\delta$, being restricted to the interval $(0,1)$, is always less than $1$.
Therefore, we have that $\HypInv(k,m,\delta,M)$ is increasing for $k < m$.

It is straightforward to adapt this approach to show monotonicity in the other parameters too.
Indeed, to show the decreasingness in $m$, simply take $m$ and $m-1$ instead of $k$ and $k+1$ and repeat the proof.
Similarly, for $\delta$, one instead takes $\delta$ and $\delta'$ such that $\delta > \delta'$ to show monotonic decreasingness; for $M$, just take $M$ and $M+1$ to prove monotonic increasingness.

Now, we refine our approach to show that the monotonicity in $k$ is in fact strict.
Fix $m$, $\delta$ and $M$.
Let $K \eqdef \HypInv(k,m,\delta,M)$ and let $K'\eqdef \HypInv(k+1,m,\delta,M)$.
To show strict increasingness, we must show that $K < K'$.
First, note that by definition of the pseudo inverse, we have that $\Hyp(k,m,K,M)\le\delta$ but $\Hyp(k,m,K-1,\delta) > \delta$, and similarly, $\Hyp(k+1,m,K',M) \le \delta$ but $\Hyp(k+1,m,K'-1,\delta) > \delta$.
Therefore, combining both facts, we have
\begin{equation}\label{eq:app1}
    \Hyp(k+1,m,K',M) \le \delta < \Hyp(k,m,K-1,M).
\end{equation}
Thus, if $\Hyp(k+1,m,K,M) \ge \Hyp(k,m,K-1,M)$, the result will follow immediately.
Indeed, if this inequality is true, as will be shown in the next paragraph, using it in Equation~\eqref{eq:app1} yields
\begin{equation*}
    \Hyp(k+1,m,K',M) < \Hyp(k+1,m,K,M).
\end{equation*}
Since $\Hyp(k,m,K,M)$ is strictly decreasing in $K$, we have that $K < K'$ as desired.

The proof of the inequality goes as follows.
Start from $\Hyp(k+1,m,K,M)$, and then subtract $\Hyp(k,m,K,M)$ from it and multiply by $\binom{M}{m}$ to obtain
\begin{align*}
    \binom{M}{m} (\Hyp(k+1&, m,K,M) - \Hyp(k,m,K,M))\\
    &= \binom{K}{k+1} \binom{M-K}{m-k-1}\\
    &= \frac{K}{k+1} \binom{K-1}{k} \binom{M-K}{m-k-1}\\
    &= \frac{K}{k+1} \binom{M}{m}\pr{ \Hyp(k,m,K-1,M) - \Hyp(k,m,K,M) },
\end{align*}
where we used Berkopec's identity at the last line.
Since $K \geq k+1$, the inequality must hold.
\end{proof}

\subsection{Computing the hypergeometric tail pseudo-inverse}
\label{app:computing_HypInv}

Lacking a closed form for the hypergeometric tail pseudo-inverse forces us to resort to an algorithm to evaluate it.
We give two algorithms with different running time for that purpose.

\subsubsection{A bisection algorithm}

The first algorithm we propose is a modification to the classical bisection algorithm, adapted to handle the discrete nature of the hypergeometric tail function, and is presented in Algorithm~\ref{algo:bisection_hyp_inv}.

\def\low{\textnormal{low}}
\def\high{\textnormal{high}}
\def\mid{\textnormal{mid}}
\renewcommand\AlCapFnt{\normalfont}

\begin{algorithm2e}[ht]
\caption{Bisection for the hypergeometric tail pseudo-inverse} \label{algo:bisection_hyp_inv}
\Input{Integers $k, m, M$ such that $0 \le k \le m \le M$, and $\delta \in (0, 1)$}
\Output{$\HypInv(k, m, \delta, M)$}
\BlankLine
$K_\low \assign k$ \;

$K_\high \assign M-m+k+1$ \;

$K_\mid \assign \left\lceil \frac{K_\low + K_\high}{2} \right\rceil$ \;

\While{$K_\high - K_\low > 1$}{
    \uIf{$\Hyp(k, m, K_\mid, M) > \delta$}{
        $K_\low \assign K_\mid$ \;
    }\Else{
        $K_\high \assign K_\mid$ \;
    }
    $K_\mid \assign \left\lceil \frac{K_\low + K_\high}{2} \right\rceil$ \;
    }
\Return{$K_\high$}

\end{algorithm2e}

The while loop will iterate about $\log_2(M-m)$ times.
Moreover, using Definition~\ref{def:app_hyp_tail} of the hypergeometric tail function to evaluate $\Hyp(k, m, K_\mid, M)$ requires summing $k$ terms, making the algorithm of order $\Theta(k\log(M-m))$, regardless of $\delta$.
This property makes the algorithm quick and reliable when $k$ is small.

\subsubsection{A linear search algorithm}

The bisection algorithm relied on the standard Definition~\ref{def:app_hyp_tail} of the hypergeometric tail.
However, one can use Berkopec's identity of Proposition~\ref{prop:berkopec} to obtain an algorithm with a different asymptotic rate.

The idea is to evaluate $\Hyp(k,m,K,M)$ incrementally using Berkopec's identity by adding one term of the sum at a time.
At each step, the sum increases until it exceeds the value of $\delta$.
To avoid computing large binomial coefficients, one can simply update the term to be added at each step.
This is what is done in Algorithm~\ref{algo:linear_hyp_inv}, where $B$ is the whole sum representing $\Hyp(k,m,K,M)$ and $b$ is the individual term.
The algorithm starts the sum of Proposition~\ref{prop:berkopec} at $J=K=M-m+k$ and decreases $K$ by one at each step, adding exactly one term to the sum.

\begin{algorithm2e}[h]
\caption{\noindent Linear search for the hypergeometric tail pseudo-inverse} \label{algo:linear_hyp_inv}
\Input{Integers $k, m, M$ such that $0 \le k \le m \le M$, and $\delta \in (0, 1)$}
\Output{$\HypInv(k, m, \delta, M)$}
\BlankLine
$K \assign M - m + k$ \;
$b \assign \binom{K}{k} \binom{M - K - 1}{M - K - m + k}$ \hspace{10pt} (Individual term of the sum) \;
$B \assign b$ \hspace{10pt} (Partial Berkopec sum) \;
\While{$B/ \binom{M}{m} \le \delta$ {\rm \bf and} $K > k$}{
    $b \assign b \frac{(K-k)(M-K)}{K(M-K-m+k+1)}$\hspace{10pt} (Update term of the sum)\;
    $B \assign B + b$ \hspace{10pt} (Update Berkopec sum) \;
    $K \assign K - 1$ \;
    }
\Return{$K + 1$}
\end{algorithm2e}

The running time of this algorithm varies between $\Omega(1)$ and $O(M-m)$ depending on $k$, $m$, $\delta$ and $M$.
Hence, depending on the situation, Algorithm~\ref{algo:linear_hyp_inv} might be faster than Algorithm~\ref{algo:bisection_hyp_inv}.
However, when $k$ is small and $M$ is large, which happens most of the time, the bisection algorithm is probably the fastest and safest choice.

Note that this algorithm makes the choice to begin the linear search from below the threshold $\delta$, but it is also possible to slightly modify the algorithm so that it approaches $\delta$ from above.
However, this version would be less relevant in our context since $\delta$ is (very) small in all our applications.

\clearpage

\section{Langford's binominal tail inversion theorem}
\label{app:langfords_binomial_tail_inversion_theorem}

In its 2005 tutorial on prediction theory, \citeauthor{langford05} explores and compares elementary generalization bounds in the classification setting.
In particular, he presents a ``unifying'' framework for \emph{test bounds}, \ie, bounds that can only be used to evaluate the performance of a classifier on a test set, but cannot be used to learn a classifier.
In order to do so, he considers the binomial tail function
\begin{equation}
    \Bin(m, k, p) \eqdef \sum_{j=0}^{k} \binom{m}{j} {p}^{\,j} (1-p)^{m-j},
\end{equation}
which can be interpreted as the probability of making $k$ errors or less on a sample of $m$ examples if the true probability of making an error is $p \in (0,1)$.
Then, because $\Bin(m,k,p)$ is a continuous one-to-one function with respect to parameter $p$, one can define the \emph{inverse} of the binomial tail function\footnote{For our convenience, this definition differs from \cite{langford05} in that it uses the minimum instead of the maximum; but because $\Bin$ is a continuous strictly decreasing function with respect to $p$, both definitions are equivalent.} as follows:
\begin{equation}\label{def:bin_tail_inv}
    \BinInv(m, k, \delta) \eqdef
    \begin{cases}
        \min \cb{ p \in (0,1): \Bin(m,k,p) \le \delta } & \textnormal{if } k < m,\\
        1 & \textnormal{if } k=m.
    \end{cases}
\end{equation}
We handle the case $k=m$ separately, as $\Bin(m,m,p)=1$ for all $p\in (0,1)$.

With this quantity at hand, \citeauthor{langford05} states the following theorem (Theorem 3.3 in the tutorial).
\begin{theorem}[Binomial tail inversion bound]\label{thm:bin_tail_inv}
For any distribution $\D$ on $\X \times \Y$, for any classifier $h\in\H$, any integer $m>0$ and for any $\delta\in(0,1)$, we have
\begin{equation}\label{eq:theorem_bintailinv}
    \prob{S\sim \D^m}{R_\D(h) \le \BinInv(m, m R_S(h), \delta)} > 1-\delta
\end{equation}
\end{theorem}
Finally, by upper bounding $\BinInv$ in different ways, one is able to recover several other known looser bounds.

\begin{proof}
In the following, we let $k_S(h) \eqdef m R_S(h)$ be the number of errors made by the classifier $h$ on the sample.
We start from the left-hand side of Equation~\eqref{eq:theorem_bintailinv} and we decompose the probability over the various values of $k$:
\begin{align*}
\prob{S}{R_\D(h) \leq \BinInv(m, k_S(h), \delta)}
	&= \exv{S}{ \Id{R_\D(h)\leq \BinInv(m, k_S(h), \delta)} }\\
	&= \sum_{k=0}^{m} \exv{S}{ \Id{R_\D(h)\leq \BinInv(m, k, \delta)} \Id{k_S(h) = k} }\\
	&= \sum_{k=0}^{m} \prob{S}{ k_S(h) = k } \Id{R_\D(h)\leq \BinInv(m, k, \delta)}.
\end{align*}
Here, $k_S(h)$ is a binomial random variable with probability $R_\D(h)$, and therefore we have that $\prob{S}{ k_S(h) = k } = \binom{m}{k} (R_\D(h))^k (1 - R_\D(h))^{m-k}$.
This implies that 
\begin{align*}
\prob{S}{R_\D(h) \leq \BinInv(m, k_S(h), \delta)}&\\
	&\hspace{-75pt}= \sum_{k=0}^{m} \binom{m}{k} (R_\D(h))^k (1 - R_\D(h))^{m-k}\, \Id{R_\D(h)\leq \BinInv(m, k, \delta)}\\
	&\hspace{-75pt}=1 - \sum_{k=0}^{m} \binom{m}{k} (R_\D(h))^k (1 - R_\D(h))^{m-k} \Id{R_\D(h)> \BinInv(m, k, \delta)},
\end{align*}
where we have used the fact that $\Bin(m, m, p) = 1$ for all values of $p$.

Because $\BinInv(m,k,\delta)$ is strictly increasing with $k$, the indicator function implies that the sum is truncated at the largest value of $k$ such that $R_\D(h) > \BinInv(m,k,\delta)$.
Thus, define $\kappa(h) \eqdef \max\{ k:R_\D(h) > \BinInv(m,k,\delta)\}$, so that the right-hand side can be written as
\begin{align*}
\prob{S}{R_\D(h) \leq \BinInv(m, k_S(h), \delta)}
	&=1 - \sum_{k=0}^{\kappa(h)} \binom{m}{k} (R_\D(h))^k (1 - R_\D(h))^{m-k}\\
	&=1 - \Bin(m, \kappa(h), R_\D(h)).
\end{align*}

Finally, because $\Bin(m,k,p)$ is strictly decreasing with $p$ and because we have that $R_\D(h) > \BinInv(m, \kappa(h), \delta)$, it follows that
\begin{equation*}
    \Bin(m, \kappa(h), R_\D(h)) < \Bin(m, \kappa(h), \BinInv(m, \kappa(h), \delta))=\delta
\end{equation*} by definition of the inverse of the binomial tail.
Therefore, we end up with
\begin{equation*}
\prob{S}{R_\D(h) \leq \BinInv(m, k_S(h), \delta)} > 1 - \delta,
\end{equation*}
as desired
\end{proof}

Remark that every step of the proof holds with equality: the theorem is thus ``essentially perfectly tight'', as written by \cite{langford05}.
Indeed, it is impossible to be tighter without further assumptions about the distribution $\D$ at hand.

We would like to point out two critical steps of the proof which allow the approach to be perfectly tight, as opposed to other ones.
The first is the observation that one can truncate the summation over $k$ using the indicator function, which eliminates superfluous terms leads to major improvements over other bounds.
The second, which follows from the first, is to notice that the truncated sum is nothing else than the distribution tail, which can be inverted quite easily with today's computing power.

This bound is only valid for a single classifier at a time, and therefore cannot be used to select the best classifier $h$ from a hypothesis class $\H$.
To mitigate this shortcoming, \cite{langford05} extends it to the case of a countable hypothesis class using the union bound.
However, machine learning algorithms usually consider uncountably infinite hypothesis classes, a challenge this present paper aims to tackle.

\clearpage

\section{In-depth analysis of the ghost sample trade-off}
\label{app:ghost_sample_size_tradeoff}

\emph{
This appendix is an extended version of Section~\ref{ssec:ghost_sample_size_trade-off} and aims to provide more information on the ghost sample trade-off.
}

In their derivations, \cite{vapnik98} and \cite{as-93} both set the ghost sample size $m'$ equal to the number $m$ of examples in the sample.
This choice might seem like a negligible detail at first, but upon further inspection, we show that the value of the bound can be improved significantly by choosing $m'$ more carefully.

We first investigate how the value of $m'$ can affect the bound $\epsilon$ of Theorem~\ref{thm:main} and what factors influence the optimal value of $m'$.
We then discuss how to select the best $m'$ and we examine the possible gain obtained by making this choice.

The upper bound $\epsilon$ of Theorem~\ref{thm:main} presents an interesting dependence in $m'$.
Indeed, increasing $m'$ makes the factor of $\frac{1}{m'}$ shrink, while $\HypInv(k,m,\delta,M)$ decreases when $M$ grows (keeping the other three parameters fixed), and increases when $\delta$ becomes smaller, as shown in Appendix~\ref{app:hyp_tail_inv}.
Since the confidence parameter $\delta$ is divided by the growth function $\tau_\H(m+m')$, which we assume is polynomial in $m+m'$ (for $m+m'$ large enough), we can expect that at some point, increasing $m'$ becomes less beneficial than decreasing it, thus presenting a non trivial trade-off.

Due to the discrete nature of the hypergeometric tail pseudo-inverse, finding an analytic expression for the optimal value for $m'$ is unfeasible.
We are therefore forced to resort to an empirical approach to analyze the factors that influence the minimum of the bound.
To obtain a numerical value for the bound, we assume a binary classification setting and we use Sauer-Shelah's lemma to upper bound the growth function as $\tau_\H(m+m') \le \big(\frac{e(m+m')}{d}\big)^d$, where $d$ is the VC dimension of $\H$.
Using this approximation, we have that the factors that can affect the optimal value of $m'$ are $m$, $k$, $d$ and $\delta$.
To have an idea of the behavior we are dealing with, we plot in Figure~\ref{fig:bound_comp} the bound $\epsilon$ of Theorem~\ref{thm:main} against the ghost sample size $m'$ for various
representative sets of parameters as a function of $k$, $m$, $\delta$ and $d$ along with the minimum of each curve indicated by a bullet point.
We use Algorithm~\ref{algo:bisection_hyp_inv} of Appendix~\ref{app:computing_HypInv} to compute the pseudo-inverse of the hypergeometric tail.

\begin{figure}[h!]
\centering
\begin{subfigure}[t]{0.485\textwidth}
    \centering
    \includegraphics[width=\textwidth]{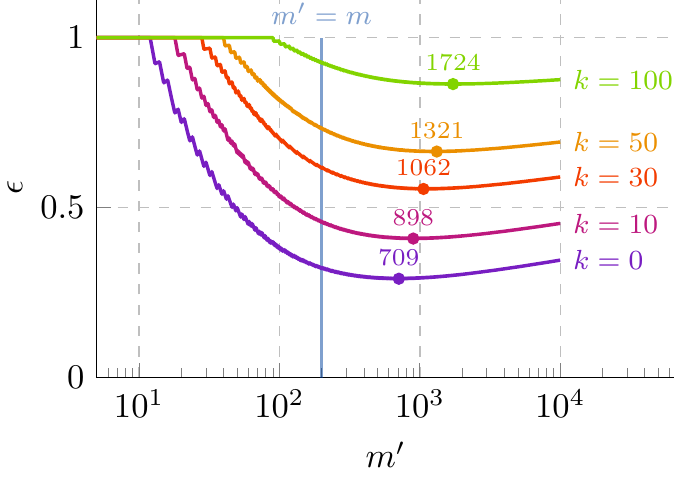}
    \caption{The influence of the number $k$ of errors on $\epsilon$. The optimal value of $m'$ increases with $k$.}
    \label{fig:comp_k}
\end{subfigure}\hfill
\begin{subfigure}[t]{0.485\textwidth}
    \centering
    \includegraphics[width=\textwidth]{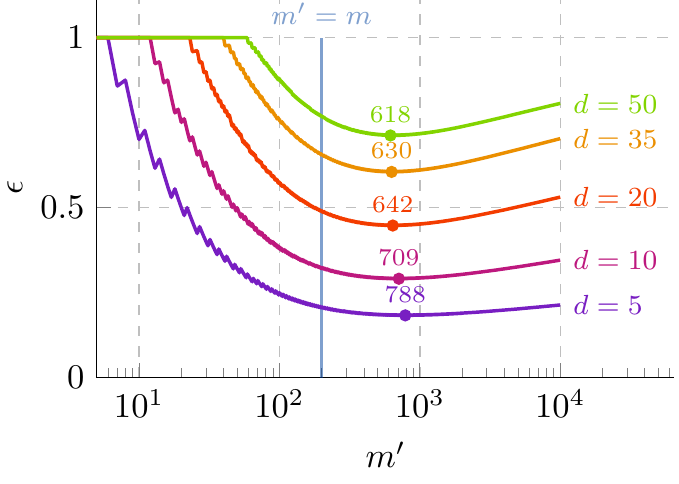}
    \caption{The influence of the VC dimension $d$ on $\epsilon$. The optimal value of $m'$ decreases with $d$.}
    \label{fig:comp_d}
\end{subfigure}

\vspace{10pt}

\begin{subfigure}[t]{0.485\textwidth}
    \centering
    \includegraphics[width=\textwidth]{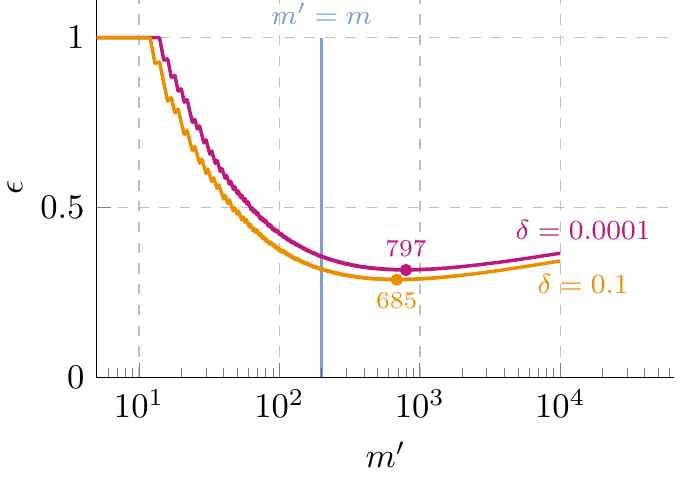}
    \caption{The influence of the confidence parameter $\delta$ on $\epsilon$. The optimal value of $m'$ decreases with $\delta$. Only two values of $\delta$ are shown to avoid clutter.}
    \label{fig:comp_delta}
\end{subfigure}\hfill
\begin{subfigure}[t]{0.485\textwidth}
    \centering
    \includegraphics[width=\textwidth]{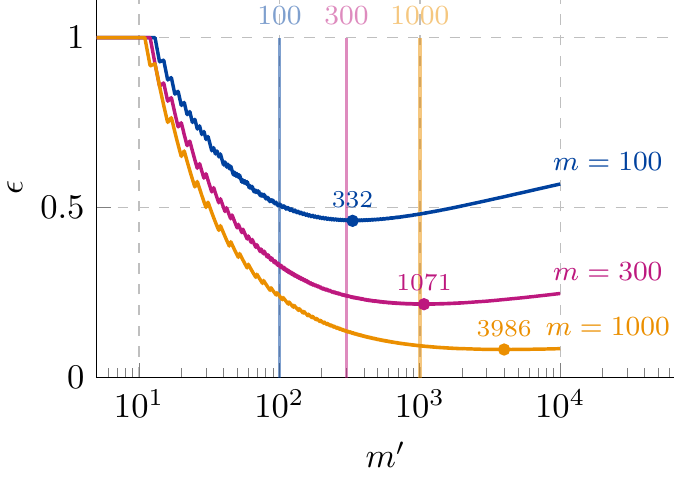}
    \caption{The influence of the number $m$ of examples on $\epsilon$. The optimal value of $m'$ increases with $m$.}
    \label{fig:comp_m}
\end{subfigure}
\caption{Upper bound $\epsilon$ as a function of $m'$, varying the four factors that influence its behavior.
The default values of the parameters are $k=0$, $d=10$, $\delta=0.05$ and $m=200$; overwritten values are specified at the right of each line plot.
The minimum of each line plot is denoted by a bullet point with the associated value of $m'$ close to it.}
\label{fig:bound_comp}
\end{figure}

Note that even though $\epsilon$ is a discrete function of $m'$, we have joined consecutive values of $\epsilon$ by line segments in Figure~\ref{fig:bound_comp} in order to ease the reading. 
Furthermore, the jagged look is due to the fact that the pseudo-inverse is a step function in its third parameter, with steps of uneven lengths.

Observe that, in all cases, the curve takes a U-shape, typical of trade-offs.
Furthermore, the position of the minimum is very far from $m$ (note that the axes have a logarithmic scale), and moves further with $k$ and $\delta$ increasing, and with $m$ and $d$ decreasing.
The parameters with the greatest impact on the optimal value of $m'$ are $m$ and $k$.

To illustrate this phenomenon, Figure~\ref{fig:mprime_best} presents the $m'$ that minimizes $\epsilon$ relative to the number $m$ of examples as a function of the risk for $\delta=0.05$, and $d=10$ or $d=35$.
While Figure~\ref{fig:mprime_best_d=10} hints at an approximately linear relation between $k$ and $m'_\textnormal{best}$, Figure~\ref{fig:mprime_best_d=35} points towards an exponential behavior, where $m'_\textnormal{best}$ can be multiple orders of magnitude greater than $m$ in some situations.

\begin{figure}[h!]
\centering
\begin{subfigure}[t]{0.485\textwidth}
    \centering
    \includegraphics[width=\textwidth]{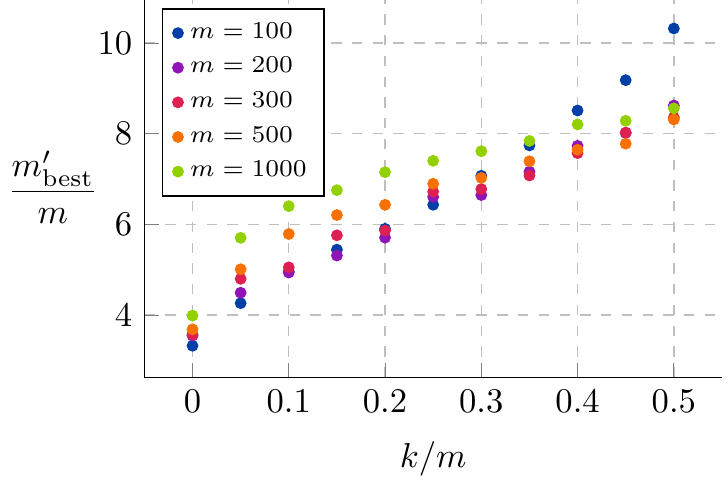}
    \caption{VC dimension set to $d=10$.}
    \label{fig:mprime_best_d=10}
\end{subfigure}\hfill
\begin{subfigure}[t]{0.485\textwidth}
    \centering
    \includegraphics[width=\textwidth]{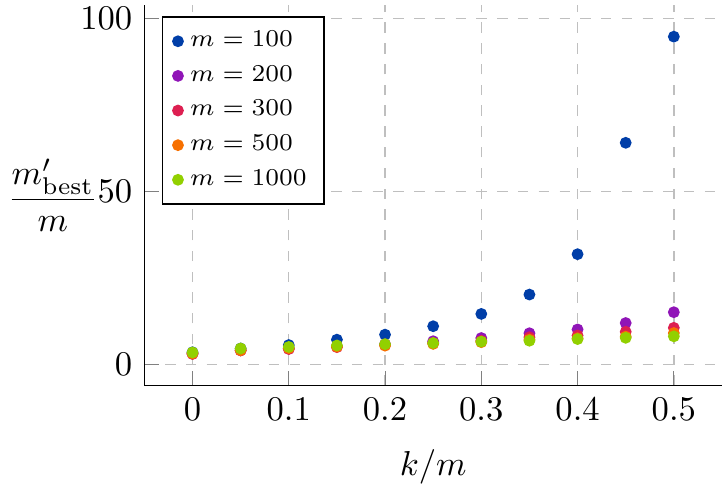}
    \caption{VC dimension set to $d=35$.}
    \label{fig:mprime_best_d=35}
\end{subfigure}
\caption{Value of $m'$ that minimizes the upper bound $\epsilon$ relative to the number $m$ of examples as a function of the risk $k/m$ for $\delta=0.05$ and two values of VC dimension, $d=10$ and $d=35$.}
\label{fig:mprime_best}
\end{figure}

The choice of the ghost sample size is subject to an important constraint we have not discussed yet.
Indeed, while it is true that the value of $m'$ that minimizes $\epsilon$ depends on the number $k$ of errors, the choice of $m'$ \emph{cannot} depend on the actual empirical risk.
As specified in the statement of Theorem~\ref{thm:main}, $m'$ must be chosen ahead of time, similarly to the number $m$ of examples, the confidence parameter $\delta$ and the hypothesis class $\H$.
In fact, letting $m'$ depend on $k$ would imply that we have a different ghost sample $S'$ for every number $k$ of errors, which is incompatible with the proof.
On the other hand, one can still choose an $m'$ for an \emph{anticipated} number of errors classifiers will make, as some sort of prior knowledge.
Of course, one might use the union bound so that the theorem holds simultaneously for all $m'$, but this could potentially lead to an increase in the bound. 

Since it is out of our current reach to get an exact analytic expression for the optimal value of $m'$, we can only find the optimal value of $m'$ numerically for a preemptively fixed set of parameters.
Unfortunately, $\epsilon$ being discrete and non-convex, one cannot use an efficient gradient descent or bisection algorithm to find the minimum: we have to compute the bound for all values of $m'$ in some interval and make a linear search to find the best value of $m'$.
In general, this optimization has to be made only once per experiment.
However, if for some reason one cannot afford to proceed with this computation, the heuristic $m'\approx 4m$ seems like a reasonable choice according to Figure~\ref{fig:mprime_best_d=10}---assuming an anticipated empirical risk around $5\%$.

As a final note, to give an idea of the gain that can be attained by choosing the best $m'$, we computed the optimal value of $m'$ for a large number of sets of parameters $(m, k/m, d, \delta)$.
We considered $m \in \cb{100, 200, 300, 500, 1000}$, $k/m \in \cb{0, 0.05, 0.1, \dots, 0.5}$, $d \in \cb{5, 10, 20, 35}$ and $\delta \in \cb{0.0001, 0.0025, 0.05, 0.1}$, for a total of 880 combinations. 
The mean relative gain (given by $1 - \epsilon(m'=m'_\text{best})/\epsilon(m'=m)$) was about $8.54\%$ with a standard deviation of $2.36\%$.
When restricting along the $k=0$ subspace, the gain went up to $10.29\%$, with a standard deviation of $1.84\%$,
which is a substantial improvement.

\clearpage

\section{The hypergeometric tail inversion relative deviation bound}
\label{app:hyptailinv_reldev}

In this section, we consider the relative deviation version of the hypergeometric tail inversion bound.
We first state and prove the theorem from which Corollary~\ref{coro:relative_deviation} stems, and we then study the impact of the ghost sample size to numerically compare this version of the bound with the original of Theorem~\ref{thm:main}.

\subsection{Proof of the hypergeometric tail inversion relative deviation bound}
\label{app:proof_reldev}
Theorem~\ref{thm:relative_deviation} states that the difference between the true risk and the empirical risk is bounded from above as a function of the true risk with high probability.

\begin{theorem}
\label{thm:relative_deviation}
Let $\H \subseteq \mathcal{Y}^\mathcal{X}$ be a hypothesis class.
Then for any distribution $\D$ on $\mathcal{X} \times \mathcal{Y}$, any confidence parameter $\delta \in (0, 1)$ and any integers $m$, $m'>0$, the following inequality holds:
\begin{equation*}
\prob{S\sim\D^m}{\forall h \in \H, R_\D(h) - R_S(h) \leq \eta(m R_S(h)) \sqrt{R_\D(h)}} > 1 - \delta,
\end{equation*}
where
\begin{equation*}
\eta(k) = \begin{cases}
    \max\cb{ \frac{1}{\sqrt{m'}}, \frac{m+m'}{m'} \frac{u(k) - \frac{k}{m}}{\sqrt{u(k)}}  } & \text{if } 0 \leq k < m \\
    0 & \text{if } k = m
\end{cases},
\end{equation*}
and
\begin{equation*}
    u(k) = \frac{1}{m+m'}\pr{ \HypInv\big(k, m, \textstyle\frac{\delta}{4 \tau_\H(m + m')}, m + m'\big) - 1}.
\end{equation*}
\end{theorem}

Notice that stating the theorem the way it is instead of in the form $\frac{R_\D(h) - R_S(h)}{\sqrt{R_\D(h)}} \le \eta(m R_S(h))$ avoids the problem of division by zero completely.
This shows that dealing with $R_\D(h) = 0$ is not as complicated as was argued by \cite{cortes2019relative}.
To prove the theorem, we will make use of the following lemma, which will be used for the ghost sample trick.

\begin{lemma}\label{lem:rel_dev}
Consider the setting of Theorem~\ref{thm:relative_deviation} and let $S$ and $S'$ be samples of sizes $m$ and $m'$ respectively such that $T \eqdef (S, S')$ is i.i.d.
Then, for any hypothesis $h\in\H$ and $\eta^2 \geq \frac{1}{m'}$, we have
\begin{equation*}
\Id{R_\D(h) - R_S(h) > \eta \sqrt{R_\D(h)}}
    < 4\prob{S'}{R_{S'}(h) - R_S(h) > \eta \sqrt{R_T(h)}}.
\end{equation*}
\end{lemma}

Note that similar results are offered by \cite{vapnik98}, \cite{as-93} and \cite{cortes2019relative}, but differ in that $m'$ is set equal to $m$.

\begin{proof} \hspace{-4pt}\textbf{of Lemma~\ref{lem:rel_dev}}
Assume $R_\D(h) - R_S(h) > \eta \sqrt{R_\D(h)}$ holds.
Then we need to show that
\begin{equation*}
    \frac{1}{4} < \prob{S'}{R_{S'}(h) - R_S(h) > \eta \sqrt{R_T(h)}}.
\end{equation*}
To do so, we will make use of the theorem of \cite{greenberg14} stating that \begin{equation}
    \frac{1}{4} < \prob{S'}{R_{S'}(h) \ge R_\D(h)},\label{eq:lemma_greenberg}
\end{equation}
provided that $R_\D(h) \ge \frac{1}{m'}$.
Let us first confirm that our situation permits the use of this result.
From our assumption, we have that $R_\D(h) - \eta \sqrt{R_\D(h)} > R_S(h) \ge 0$, which implies $R_\D(h) > \eta^2 \ge \frac{1}{m'}$ (from the definition of $\eta$ as stated in the lemma).

If we can show that we have
\begin{equation*}
    \prob{S'}{R_{S'}(h) \ge R_\D(h)} < \prob{S'}{R_{S'}(h) - R_S(h) > \eta \sqrt{R_T(h)}},
\end{equation*}
then the lemma will follow directly from \eqref{eq:lemma_greenberg}.
This inequality can be proven by showing that $R_{S'}(h) \ge R_\D(h)$ implies $R_{S'}(h) - R_S(h) > \eta \sqrt{R_T(h)}$.
Therefore, assuming $R_{S'}(h) \ge R_\D(h)$ holds, we will show that $\frac{R_{S'}(h) - R_S(h)}{\sqrt{R_T(h)}} > \eta$.

Consider the function
\begin{equation*}
    f_{\alpha, \beta}(x, y) \eqdef \frac{x - y}{\sqrt{\alpha x + \beta y}},
\end{equation*}
for $\alpha, \beta > 0$, $x > 0$ and $y \ge 0$.
Taking the partial derivatives with respect to $x$ and $y$, we have
\begin{equation*}
    \frac{\partial f_{\alpha, \beta}}{\partial x} = \frac{\alpha x + (\alpha+2\beta)y}{2(\alpha x + \beta y)^{3/2}} > 0
    \qquad \text{ and } \qquad
    \frac{\partial f_{\alpha, \beta}}{\partial y} = -\frac{(2\alpha + \beta) x + \beta y}{2(\alpha x + \beta y)^{3/2}} < 0.
\end{equation*}
This informs us that $f_{\alpha, \beta}(x,y)$ is increasing with $x$ and decreasing with $y$.

Applying this result to our original problem, we have
\begin{align*}
    \frac{R_{S'}(h) - R_S(h)}{\sqrt{R_T(h)}}
    &= \frac{R_{S'}(h) - R_S(h)}{\sqrt{\frac{m' R_{S'}(h) + m R_S(h)}{m+m'}}}\\
    &= f_{\frac{m'}{m+m'}, \frac{m}{m+m'}}(R_{S'}(h), R_S(h))\\
    &> f_{\frac{m'}{m+m'}, \frac{m}{m+m'}}(R_\D(h), R_\D(h)-\eta\sqrt{R_\D(h)})\\
    &= \frac{\eta \sqrt{R_\D(h)}}{\sqrt{ R_\D(h) - \frac{m}{m+m'}\eta\sqrt{R_\D(h)} }}\\
    &> \eta,
\end{align*}
where we have used both hypotheses to go from the second to the third line.
Note that the radicand at the fourth line is always greater than zero by our assumption that $R_\D(h)- \eta \sqrt{R_\D(h)} > 0$.
\end{proof}

We are now ready to prove the relative deviation bound.
Note that many steps of the proof are similar to the steps of our main theorem's proof.
As such, we do not delve into details and justifications here unless the proof differs.

\begin{proof}\hspace{0pt}\textbf{of Theorem~\ref{thm:relative_deviation}}
We will prove the result by proving that the opposite event occurs with probability less than or equal to $\delta$, \ie,
\begin{equation*}
\prob{S\sim\D^m}{\exists h \in \H, R_\D(h) - R_S(h) > \eta(m R_S(h)) \sqrt{R_\D(h)}} \leq \delta.
\end{equation*}

We begin with the ghost sample trick, using Lemma~\ref{lem:rel_dev} to replace the dependence on the true risk by the empirical risk of a second sample.
Notice that the definition of $\eta(k)$ in Theorem~\ref{thm:relative_deviation} makes the probability zero whenever $k=m$ and otherwise we have that $\eta(k)^2 \geq \frac{1}{m'}$ so that the requirements of the lemma are always satisfied.

Starting with the left-hand side, we have
\begin{align*}
&\prob{S\sim\D^m}{\exists h \in \H, R_\D(h) - R_S(h) > \eta(m R_S(h)) \sqrt{R_\D(h)}}\\
    &\hspace{40pt} = \exv{S}{\sup_{h\in\H} \Id{ R_\D(h) - R_S(h) > \eta(m R_S(h)) \sqrt{R_\D(h)} }}\\
    &\hspace{40pt} \le 4 \exv{S}{\sup_{h\in\H} \prob{S'}{R_{S'}(h) - R_S(h) > \eta(m R_S(h)) \sqrt{R_T(h)}} }\\
    &\hspace{40pt} \le 4 \exv{S, S'}{\sup_{h\in\H} \Id{R_{S'}(h) - R_S(h) > \eta(m R_S(h)) \sqrt{R_T(h)}} }\\
    &\hspace{40pt} = 4 \exv{T}{\max_{h\in\H|_T} \Id{R_{S'}(h) - R_S(h) > \eta(m R_S(h)) \sqrt{R_T(h)}} }\\
    &\hspace{40pt} = 4 \exv{T}{\max_{h\in\H|_T} \Id{R_{T}(h) - R_S(h) > \frac{m'}{m+m'} \eta(m R_S(h)) \sqrt{R_T(h)}} },
\end{align*}
where $\H|_T$ is the restriction of $\H$ to the sample $T$, as defined in the proof of Theorem~\ref{thm:main}, and at the last line we have used the relation $R_T(h) = \frac{m R_S(h) + m' R_{S'}(h)}{m+m'}$ to express the inequality in terms of $R_T(h)$.

We can now proceed to the symmetrization trick, using the same notation as before, letting $\sigma \in \Sigma \eqdef \{ \sigma \subseteq [m + m'] \: : \: \abs{\sigma} = m \}$ denote a set of $m$ indices, and defining $T(\sigma)$ as the examples of $T$ pointed by the indices of $\sigma$, such that $S=T([m])$.
Again assuming the uniform distribution over $\Sigma$, we have
\begin{align}
&\prob{S\sim\D^m}{\exists h \in \H, R_\D(h) - R_S(h) > \eta(m R_S(h)) \sqrt{R_\D(h)}}\nonumber\\
    &\hspace{20pt} \le 4 \exv{T}{\max_{h\in\H|_T} \Id{R_{T}(h) - R_{T([m])}(h) > \frac{m'}{m+m'} \eta(m R_S(h)) \sqrt{R_T(h)}} }\nonumber\\
    &\hspace{20pt} \le 4 \exv{T}{\sum_{h\in\H|_T} \Id{R_{T}(h) - R_{T([m])}(h) > \frac{m'}{m+m'} \eta(m R_S(h)) \sqrt{R_T(h)}} }\nonumber\\
    &\hspace{20pt} = 4 \exv{\sigma}{\exv{T}{\sum_{h\in\H|_T} \Id{R_{T}(h) - R_{T(\sigma)}(h) > \frac{m'}{m+m'} \eta(m R_S(h)) \sqrt{R_T(h)}} } }\nonumber\\
    &\hspace{20pt} = 4 \exv{\sigma}{\exv{T}{\sum_{h\in\H|_T} \sum_{k=0}^{m-1}\Id{R_{T}(h) - \frac{k}{m} > \frac{m'}{m+m'} \eta(k) \sqrt{R_T(h)}}\Id{ R_{T(\sigma)}(h) = \frac{k}{m}} } }\nonumber\\
    &\hspace{20pt} = 4 \exv{T}{\sum_{h\in\H|_T} \sum_{k=0}^{m-1}\Id{R_{T}(h) - \frac{k}{m} > \frac{m'}{m+m'} \eta(k) \sqrt{R_T(h)}} \prob{\sigma}{ R_{T(\sigma)}(h) = \frac{k}{m}} }.\label{eq:reldev5}
\end{align}

Before going further, let us take a closer look at the inequality inside the indicator function.
We will show that the expression of $\eta(k)$ of Theorem~\ref{thm:relative_deviation} implies
\begin{equation*}
    \Id{R_{T}(h)\hspace{-1pt} -\hspace{-1pt} \frac{k}{m} \!>\! \frac{m'}{m+m'} \eta(k) \sqrt{R_T(h)}}
    \hspace{-1pt}\le\hspace{-1pt} \Id{R_T(h) \!\ge\! \frac{1}{m+m'}\HypInv\big(k, m, \textstyle\frac{\delta}{4 \tau_\H(m + m')}, m + m'\big)\! }\!,
\end{equation*}
or in other words, that the statement inside the left indicator function implies the statement inside the right indicator function.

Therefore, assume that 
\begin{equation}\label{eq:reldev1}
    R_{T} - \frac{k}{m} > \frac{m'}{m+m'} \eta \sqrt{R_T}
\end{equation}
is true, where we mute dependencies to alleviate the notation.
As a first step, we transform the inequality into a quadratic form by squaring both sides to get
\begin{equation}\label{eq:reldev2}
    R_T^2 - R_T\pr{2\frac{k}{m}+\frac{m'^2}{(m+m')^2} \eta^2} + \frac{k^2}{m^2} > 0.
\end{equation}
According to the fundamental theorem of algebra, we can factor the left-hand side to obtain
\begin{equation}
   (R_T - r_-)(R_T - r_+) > 0,\label{eq:reldev3}
\end{equation}
where 
\begin{equation}\label{eq:reldev4}
    r_{\pm} = \frac{k}{m} + \frac{\eta^2}{2}\frac{m'^2}{(m+m')^2} \pr{ 1 \pm \sqrt{1 + \frac{4k}{m}\frac{(m+m')^2}{m'^2 \eta^2} } } 
\end{equation}
are the roots of the second degree polynomial on the left-hand side of Equation~\eqref{eq:reldev2}.
This polynomial is an upward parabola, hence Inequality~\eqref{eq:reldev3} is satisfied if $R_T < r_-$ or $R_T > r_+$.
However, one must reject the former case because $r_- < \frac{k}{m}$, which is in contradiction with the original Inequality~\eqref{eq:reldev1}.
Thus we have shown that Inequality~\eqref{eq:reldev1} is true if and only if $R_T > r_+$.

We now want to show that the choice of $\eta(k)$ produces the desired result.
To do so, let us write $r_+ = r_+(\eta)$ as a function of $\eta$ and observe that $r_+$ is increasing with $\eta$.
This can be seen by distributing $\eta$ inside the square root in Equation~\eqref{eq:reldev4}.
Using the definition of $\eta(k)$ in the statement of Theorem~\ref{thm:relative_deviation}, keeping in mind that $k\neq m$, the following series of inequalities holds.
\begin{align*}
    R_T(h)
    &> r_+(\eta(k))\\
    &= r_+\pr{\max\cb{ \textstyle \frac{1}{\sqrt{m'}},  \frac{m+m'}{m'}\frac{u(k) - \frac{k}{m}}{\sqrt{u(k)}}}}\\
    &= \max\cb{\textstyle  r_+\pr{\frac{1}{\sqrt{m'}}}, r_+\pr{ \frac{m+m'}{m'}\frac{u(k) - \frac{k}{m}}{\sqrt{u(k)}} } }\\
    &\ge r_+\pr{\textstyle  \frac{m+m'}{m'}\frac{u(k) - \frac{k}{m}}{\sqrt{u(k)}}}\\
    &= u(k)
\end{align*}
where the third line follows from the monotonic increase in $\eta$ of $r_+$, and the last line holds because $\eta(k) = \frac{m+m'}{m'}\frac{u(k) - \frac{k}{m}}{\sqrt{u(k)}}$ is a solution to the equation $r_+(\eta(k)) = u(k)$.
This can be seen by setting $r_+$ equal to $u(k)$ in Equation~\eqref{eq:reldev4} and then solving for $\eta$ (or even more simply by starting from Inequality~\eqref{eq:reldev1}).

We have gathered so far that Inequality~\eqref{eq:reldev1} implies that we have
\begin{equation*}
    R_T > u(k) = \frac{1}{m+m'}\pr{\HypInv\big(k, m, \textstyle\frac{\delta}{4 \tau_\H(m + m')}, m + m'\big)-1}.
\end{equation*}
To conclude, notice that both $(m+m')R_T$ and $\HypInv$ must be integers.
Hence, the strict inequality can be rewritten into the following equivalent inequality
\begin{equation*}
    R_T \ge \frac{1}{m+m'}\HypInv\big(k, m, \textstyle\frac{\delta}{4 \tau_\H(m + m')}, m + m'\big),
\end{equation*}
as desired.

Getting back to our original problem, the indicator function of Equation~\eqref{eq:reldev5} can be upper bounded with
\begin{equation}
    \Id{R_{T}(h)\ge \frac{1}{m+m'}\HypInv\big(k, m, \textstyle\frac{\delta}{4 \tau_\H(m + m')}, m + m'\big) }.\label{eq:reldev6}
\end{equation}
But since $\HypInv$ is an increasing function in $k$ as proven in Lemma~\ref{lem:monotonicity_hyp_tail_inv} of Appendix~\ref{app:hyp_tail_inv}, the indicator function will truncate the sum over $k$ at some point.
Let $\kappa(h)$ be the largest value of $k$ such that expression~\eqref{eq:reldev6} equals one.
Using this fact in Equation~\eqref{eq:reldev5}, we have 
\begin{align*}
&\prob{S\sim\D^m}{\exists h \in \H, R_\D(h) - R_S(h) > \eta(m R_S(h)) \sqrt{R_\D(h)}}\\
    &\hspace{40pt} < 4 \exv{T}{\sum_{h\in\H|_T} \sum_{k=0}^{\kappa(h)} \hyp(k,m,(m+m')R_T(h),m+m') }\\
    &\hspace{40pt} = 4 \exv{T}{\sum_{h\in\H|_T} \Hyp(\kappa(h),m,(m+m')R_T(h),m+m') }\\
    &\hspace{40pt} \le 4 \exv{T}{\sum_{h\in\H|_T} \Hyp(\kappa(h),m,\HypInv\big(\kappa(h), m, \textstyle\frac{\delta}{4 \tau_\H(m + m')}, m + m'\big),m+m') }\\
    &\hspace{40pt} \le 4 \exv{T}{\sum_{h\in\H|_T} \frac{\delta}{4 \tau_\H(m + m')} }\\
    &\hspace{40pt} \le \max_T \sum_{h\in\H|_T} \frac{\delta}{\tau_\H(m + m')}\\
    &\hspace{40pt} = \delta,
\end{align*}
where we made use of the fact that $\Hyp$ is decreasing in its third parameter as proven in Lemma~\ref{lem:monotonicity_hyp_tail} of Appendix~\ref{app:hyp_tail} at the fourth line, and we used the definition of the pseudo-inverse at the fifth line.
This concludes the proof.
\end{proof}

The proof of Theorem~\ref{thm:relative_deviation} closely follows that of Theorem~\ref{thm:main}.
In fact, almost all places where the bound is loosened are identical---except for the ghost sample trick.
Indeed, we replaced the result of \cite{greenberg14} by Lemma~\ref{lem:rel_dev} in the proof.
But Lemma~\ref{lem:rel_dev} relies on this former result in addition to introducing more looseness.
Therefore, one can only expect that Corollary~\ref{coro:relative_deviation} is less tight than our main theorem.
We confirm this phenomenon empirically in Section~\ref{app:numerical_comparison}.

\subsection{Ghost sample trade-off for the hypergeometric tail inversion relative deviation bound}
\label{app:ghost_sample_tradeoff_hti-rd}

In this section, we consider the ghost sample size trade-off for the relative deviation version (Corollary~\ref{coro:relative_deviation}) of the new bound.
We consider the exact same setup as in Section~\ref{ssec:ghost_sample_size_trade-off}, using Sauer-Shelah's lemma for binary classification, and we examine the influence of the parameters $k$, $m$, $\delta$ and $d$ on the optimal value of $m'$.
Figure~\ref{fig:bound_comp_reldev} reproduces Figure~\ref{fig:bound_comp} but for the relative deviation bound version.

\begin{figure}[h]
\centering
\begin{subfigure}[t]{0.485\textwidth}
    \centering
    \includegraphics[width=\textwidth]{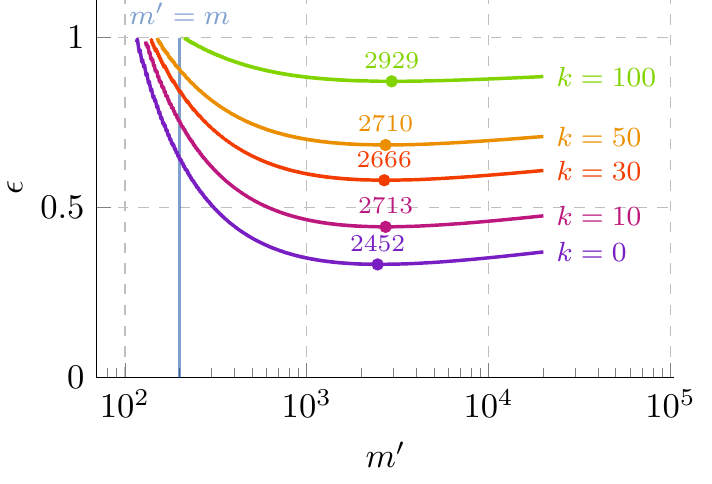}
    \caption{The influence of the number $k$ of errors on $\epsilon$.}
    \label{fig:comp_reldev_k}
\end{subfigure}\hfill
\begin{subfigure}[t]{0.485\textwidth}
    \centering
    \includegraphics[width=\textwidth]{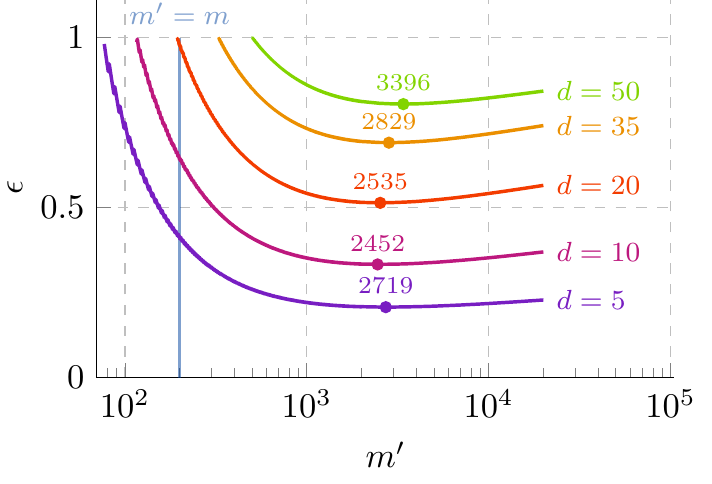}
    \caption{The influence of the VC dimension $d$ on $\epsilon$.}
    \label{fig:comp_reldev_d}
\end{subfigure}

\vspace{10pt}

\begin{subfigure}[t]{0.485\textwidth}
    \centering
    \includegraphics[width=\textwidth]{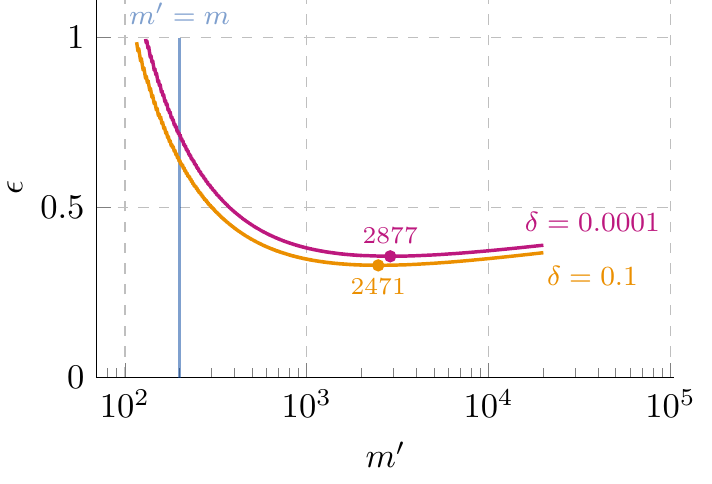}
    \caption{The influence of the confidence parameter $\delta$ on $\epsilon$. Only two values of $\delta$ are shown to avoid clutter.}
    \label{fig:comp_reldev_delta}
\end{subfigure}\hfill
\begin{subfigure}[t]{0.485\textwidth}
    \centering
    \includegraphics[width=\textwidth]{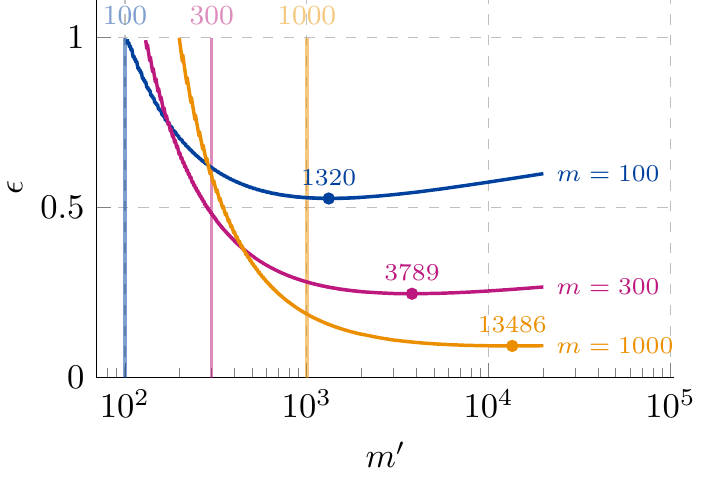}
    \caption{The influence of the number $m$ of examples on $\epsilon$.}
    \label{fig:comp_reldev_m}
\end{subfigure}
\caption{Upper bound of Corollary~\ref{coro:relative_deviation} as a function of $m'$ varying the four factors that influence its behavior.
The default values of the parameters are $k=0$, $d=10$, $\delta=0.05$ and $m=200$; overwritten values are specified at the right of each line plot.
The minimum of each line plot is denoted by a bullet point with the associated value of $m'$ close to it.}
\label{fig:bound_comp_reldev}
\end{figure}

At first glance, we can see that the optimal value of $m'$ is very large, larger than for the non-relative deviation bound.
Indeed, the minima of the curves seem to occur for $m'$ between $12m$ and $17m$, as opposed to between $3m$ and $8m$.
Furthermore, the gain to be made by selecting the optimal value of $m'$ is much larger than before.
This indicates that choosing $m'$ correctly is even more important for this new bound.

\subsection{Numerical comparison}
\label{app:numerical_comparison}

In this section, we give a numerical comparison between the hypergeometric tail inversion bound (Theorem~\ref{thm:main}, denoted by HTI) and its relative deviation counterpart (Corollary~\ref{coro:relative_deviation}, denoted by HTI-RD).
We assume binary classification so that we can use Sauer-Shelah's lemma.
We set the parameters equal to $m=1000$, $\delta=0.05$ and $d=20$.
Figure~\ref{fig:rd_bounds_comparison} presents both bounds for an $m'$ optimized for a risk of $0$ (left figure) and a risk of $0.8$ (right figure), as well as their non-optimized version where $m'=m$.
The HTI bound is presented in solid lines, while the HTI-RD is dashed.
The optimal values of $m'$ that were found for a risk of 0 are $m'=3611$ for the HTI bound and $m'=12,462$ for the HTI-RD bound.
When the risk is $0.8$, the values of $m'$ are $m'=12,851$ for the HTI bound and $m'=15,766$ for the HTI-RD bound.

Notice that in Figure~\ref{fig:rd_comp_risk_k=0}, the HTI\textsubscript{opti} bound is slightly better than the HTI-RD\textsubscript{opti} for low values of risk, up until the risk is about $0.5$, where the relative deviation bound becomes better than the HTI bound.
This may seem to contradict our argument that the HTI bound is tighter because less approximations are made.
However, this argument only holds when comparing the bounds for a risk near the risk for which the $m'$ were optimized.
In fact, when we optimize $m'$ for a high risk (for example for a risk of $0.8$), the HTI bound is always better than the HTI-RD, as can be seen in Figure~\ref{fig:rd_comp_risk_k=800}.
In any case, the difference between both bounds becomes negligible when they are optimized.

As a final observation one can make from Figure~\ref{fig:rd_bounds_comparison}, notice that the difference between the optimized and the non-optimized versions of the HTI-RD bound is considerably large, as we noted in the previous section.
This is particularly marked for low values of risk when compared with the HTI bound.

\begin{figure}[h]
\centering
\begin{subfigure}[t]{0.485\textwidth}
    \centering
    \includegraphics[width=\textwidth]{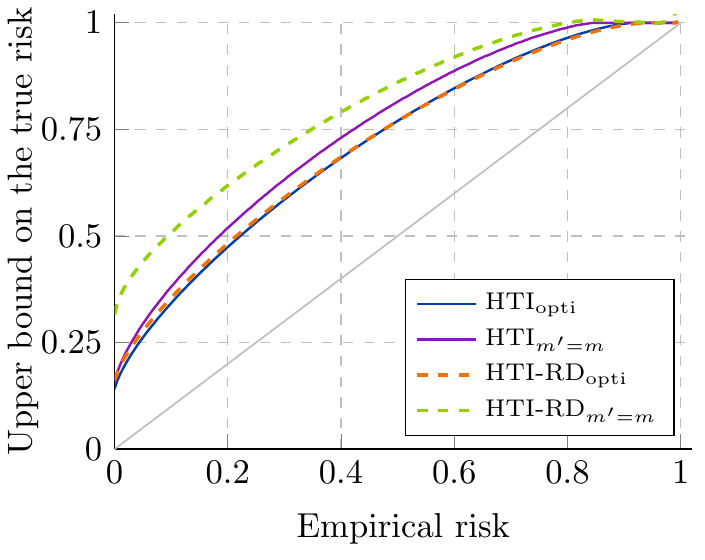}
    \caption{Bounds as functions of the empirical risk. The gray line corresponds to the empirical risk. Optimized bounds are for $k=0$.}
    \label{fig:rd_comp_risk_k=0}
\end{subfigure}\hfill
\begin{subfigure}[t]{0.485\textwidth}
    \centering
    \includegraphics[width=\textwidth]{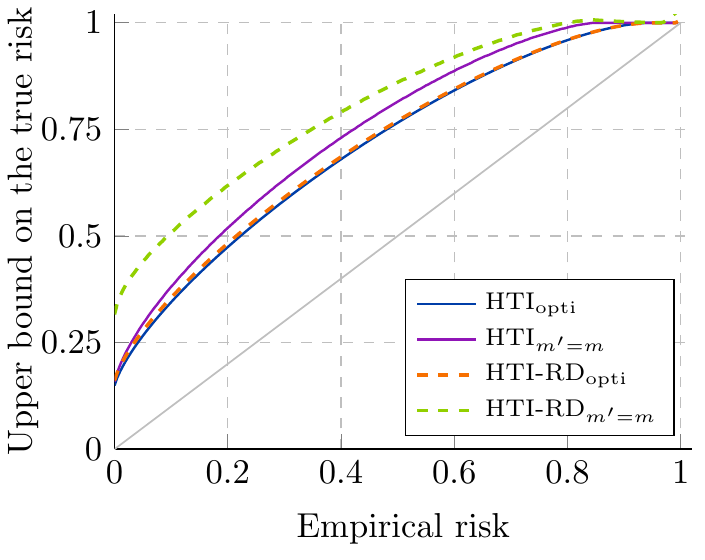}
    \caption{Bounds as functions of the empirical risk. The gray line corresponds to the empirical risk. Optimized bounds are for $k=800$.}
    \label{fig:rd_comp_risk_k=800}
\end{subfigure}
\caption{Comparison of the new hypergeometric tail inversion bound with its relative deviation version.
Bounds are plotted as functions of the risk for $m=1000$ examples with the confidence parameter $\delta$ equal to $0.05$ and the VC dimension set to $d=20$.
On the left, the optimized bounds are optimized for a risk of $0$, and on the right for a risk of $0.8$.}
\label{fig:rd_bounds_comparison}
\end{figure}

\clearpage

\section{Proof of the hypergeometric tail inversion margin bound}
\label{app:margin_bound}

In this Appendix, we prove Theorem~\ref{thm:margin_bound}, which bounds the risk of margin classifiers.
The bound is inspired from the work of \citet{anthonybartlett99}.

We assume the set-up of Section~\ref{ssec:margin_classifiers}, where classifiers are of the form 
\begin{equation*}
    h_\f(\x) = \argmax_{i\in[n]} f_i(\x),
\end{equation*}
with $\f \in \F \subseteq ([0,1]^n)^\X$ a vector-valued function.
We define the margin\footnote{Note that we have adapted the setting of \cite{anthonybartlett99} to apply to multiclass problems, and therefore the definition of the margin is different.}
of a classifier $h_\f$ as
\begin{equation*}
    \mu(\f(\x), y) \eqdef f_y(\x) - \max_{j\neq y} f_j(\x).
\end{equation*}

The mixed $L_p$-$L_\infty$ metric $\ell$ on $\reals^{n \times m}$ we consider is explicitly defined as
\begin{equation*}
    \ell(A, B) = \max_{i \in [m]} \norm{\mathbf{a}_i - \mathbf{b}_i}_p
\end{equation*}
where $\mathbf{a}_i$ and $\mathbf{b}_i$ are the $i$-th columns of $A$ and $B$, respectively, and $p > 0$.

We additionally define the true margin risk $R^\gamma_\D(\f)$ as the probability that the margin of a classifier is smaller than a margin parameter $\gamma$:
\begin{equation}\label{eq:def_margin_risk}
    R^\gamma_\D(\f) = \prob{(\x,y)}{\mu(\f(\x), y) < \gamma}.
\end{equation}
Since we have $\Id{ h(\x) \neq y } \le \Id{ \mu(\f(\x), y) \le 0 }$ (with strict inequality only on tie breaks), it implies that $R_\D(\f) \le R^\gamma_\D(\f)$ for any $\gamma > 0$.

Below, we repeat the theorem for convenience.

{\noindent\bf Theorem~\ref{thm:margin_bound}.}
\textit{
Let $\F \subseteq ([0,1]^n)^\X$ be a vector-valued function class, $h_\f(\x) = \argmax_{i\in[n]} f_i(\x)$ be a classifier and $R^\gamma_S(\f) \eqdef \exv{S}{\Id{\mu(\f(\x)),y) < \gamma}}$ be the empirical margin risk.
Then for any distribution $\D$ on $\X \times \Y$, any confidence parameter $\delta \in (0, 1)$, any margin $\gamma \in (0,1)$, any $p\ge1$ and any integers $m$ and $m' > 0$,
\begin{equation*}
\prob{S\sim \D^m}{ \forall \f \in \F, R_\D(h_\f) \leq \epsilon(m R^\gamma_S(\f), m, \delta) } > 1 - \delta,
\end{equation*}
where $\epsilon(k,m,\delta) = 1$ if $k=m$ and otherwise
\begin{equation*}
\epsilon(k, m, \delta) = \frac{1}{m'} \max\cb{ 1, \HypInv\big(k, \frac{\delta}{4 \N^p_\F\pr{\frac{\gamma}{2^{2 - 1/p}}, m + m'}}, m, m + m'\big) - 1 - k }.
\end{equation*}
}

Note that we ask for the metric of the cover to have a $p \ge 1$ rather than greater than 0.
It is possible to let $p \in (0,1)$, by considering a $\frac{\gamma}{2}$-cover instead of a $\gammacover$-cover, but this can only worsen the bound.
The details are discussed in a footnote of the proof.

Many steps of the proof are similar to the proof of Theorem~\ref{thm:main}.
Therefore, we shall not justify all steps and invite the reader to consult the aforementioned proof if needed.

\begin{proof}
As in the proof of Theorem~\ref{thm:main}, we will prove the equivalent statement:
\begin{equation*}
\prob{S}{ \exists \f \in \F, R_\D(h_\f) > \epsilon(m R^\gamma_S(\f), m, \delta) } \le \delta.
\end{equation*}

We first proceed with the ghost sample trick using the same result of \cite{greenberg14} as before.
It states that $\Id{R_\D > \epsilon(k)} < 4\prob{S'}{R_{S'}(h)> \epsilon(k)}$, provided that $\epsilon(k) \ge \frac{1}{m'}$, which holds according to our definition of $\epsilon$.
We have, muting dependencies,
\begin{align*}
\prob{S}{ \exists \f \in \F, R_\D(h_\f) > \epsilon(m R^\gamma_S(\f)) }
&= \exv{S}{ \sup_{\f\in\F} \Id{R_\D(h_\f) > \epsilon(m R^\gamma_S(\f))} }\\
&< \exv{S}{ \sup_{\f\in\F} 4 \prob{S'}{R_{S'}(\f) > \epsilon(m R^\gamma_S(\f))} }\\
&< 4 \exv{S,S'}{ \sup_{\f\in\F} \Id{R_{S'}(\f) > \epsilon(m R^\gamma_S(\f))} }.
\end{align*}

Next, we continue with the symmetrization trick, using the same notation as before, letting $\sigma \in \Sigma \eqdef \{ \sigma \subseteq [m + m'] \: : \: \abs{\sigma} = m \}$ denote a set of $m$ indices, and defining $T(\sigma)$ as the examples of $T$ with index in $\sigma$, which means $S=T([m])$ and $S'=T([m+m']-[m])$.
Again assuming the uniform distribution over $\Sigma$, we have
\begin{align*}
\prob{S}{ \exists \f \in \F, R_\D(h_\f) > \epsilon(m R^\gamma_S(\f)) }
&< 4 \exv{T}{ \sup_{\f\in\F} \Id{R_{T([m+m']-[m])}(\f) > \epsilon(m R^\gamma_{T([m])}(\f))} }\\
&= 4 \exv{T}{ \sup_{\f\in\F|_T} \Id{R_{T([m+m']-[m])}(\f) > \epsilon(m R^\gamma_{T([m])}(\f))} }\\
&= 4 \exv{\sigma}{\exv{T}{ \sup_{\f\in\F|_T} \Id{R_{T([m+m']-\sigma)}(\f) > \epsilon(m R^\gamma_{T(\sigma)}(\f))} }}.
\end{align*}

At this point, we must depart from the proof of our main theorem for two reasons: the supremum over $\F|_T$ might be infinite, and we deal with standard and $\gamma$-margin risks at the same time.

To solve the problem, we aim to replace the dependence on $\F|_T$ by a finite minimal $c$-cover $\fcover\subseteq \reals^{n\times \abs{T}}$, where $c=\frac{\gamma}{2^{2-1/p}}$.
By definition of a $c$-cover, we have that for each $\f \in \F|_T$, there exists $\fhat \in \fcover$ such that $\max_i \lVert\f(\x_i) - \fhat^i\rVert_p < c$.
To limit possible confusion, we denote by $\fhat^i$ the $n$-dimensional vector associated to example $i$ and by $\hat{f}_j^i$ the $j$-th component of $\fhat^i$.

We consider how the standard risk $R_{T([m+m']-\sigma)}(\f)$ is affected by this modification first.
Let $(\x_i, y_i)$ be any example in $T$ and define $\Delta^i \eqdef \f(\x_i) - \fhat^i$ to be the cover error of $\fhat^i$.
The condition $\argmax_j f_j(\x_i) \neq y_i$ implies $\mu(\f(\x_i), y_i) \leq 0$.
Furthermore, we have that $\max_{j\ne y_i} \hat{f}_j^i \ge \hat{f}_u^i$ is true for any $u\ne y_i$, and this holds in particular for $l \eqdef \argmax_{j \neq y_i} f_{j}(\x_i)$.
Therefore, assuming $\mu(\f(\x_i), y_i) = f_{y_i}(\x_i) - f_{l}(\x_i) \leq 0$, we have
\begin{align}
    \mu(\fhat^i, y_i)
    &= \hat{f}^i_{y_i} - \max_{j\ne y_i} \hat{f}^i_j\nonumber\\
    &\le \hat{f}^i_{y_i} - \hat{f}^i_l\nonumber\\
    &= (f_{y_i}(\x_i) - \Delta_{y_i}^i) - (f_{l}(\x_i) - \Delta_{l}^i)\nonumber\\
    &\le - \Delta_{y_i}^i + \Delta_{l}^i.\label{eq:margin_classif1}
\end{align}

We now need to upper bound the quantity $-\Delta_{y_i}^i + \Delta_{l}^i$ knowing that $\max_i\lVert\Delta^i\rVert_p<c$.
Let $\z = (z_1, z_2) \in \reals^2$ be some 2-dimensional vector.
Then we can simplify the problem by noting that
\begin{align*}
    -\Delta_{y_i}^i + \Delta_{l}^i < \sup_{\lVert\z\rVert_p < c} -z_1 + z_2.
\end{align*}
Notice that the strict inequality comes from the observation that no $\z \in \reals^2$ with $\norm{\z}_p < c$ can achieve the supremum.

One can use the constraint $\norm{\z}_p < c$ to remove one degree of freedom from the optimization problem.
Since we want to maximize $-z_1 + z_2$, we can assume without loss of generality that $z_1 \le 0$ and $z_2 \ge 0$ (the constraint $\norm{\z}_p < c$ affords us this assumption).
We have
\begin{equation*}
    c^p > \norm{\z}_p^p = (-z_1)^p + z_2^p.
\end{equation*}
Isolating $-z_1$ and substituting it in the previous inequality, we get
\begin{equation*}
    \sup_{\lVert\z\rVert_p < c} -z_1 + z_2 \le \max_{z_2} (c^p-z_2^p)^{1/p} + z_2.
\end{equation*}
Differentiating the right-hand side with respect to $z_2$ and equating to $0$, we find that the maximum occurs for $z_2 = 2^{-1/p} c$.
\footnote{Note that this is a maximum for $p \ge 1$ only; it is a minimum for $p \in (0,1)$.
For $p \in (0,1)$, the maximum occurs at $z_2=\frac{\gamma}{2}$ and $z_1=0$.
This implies that, for the proof to carry over, the minimal cover of $\F|_T$ should be a $\frac{\gamma}{2}$-cover instead of a $\gammacover$-cover, the same as when $p=1$.
Therefore, there is no gain to be made by choosing $p<1$.}
This yields, using the definition of $c$,
\begin{equation*}
\max_{z_2} (c^p-z_2^p)^{1/p} + z_2 = 2^{1 - 1/p} c = \frac{\gamma}{2}.
\end{equation*}

Using this result in our original Equation~\eqref{eq:margin_classif1}, we conclude that when $\argmax_j f_j(\x_i) \neq y_i$ is true, then we must have $\mu(\fhat^i,y_i) < \frac{\gamma}{2}$.
Therefore, according to the margin risk definition~\eqref{eq:def_margin_risk}, we have
\begin{equation*}
    R_{T([m+m']-\sigma)}(\f) \le R^{\frac{\gamma}{2}}_{T([m+m']-\sigma)}(\fhat).
\end{equation*}

Next, to address the margin risk $R^\gamma_{T(\sigma)}(\f)$, we instead assume that $\mu(\fhat_i, y_i) < \frac{\gamma}{2}$ and we define $l \eqdef \argmax_{j \ne y_i} \hat{f}_j^i$.
Using the same reasoning as before, we have
\begin{align*}
    \mu(\f(\x_i), y_i)
    &= f_{y_i}(\x_i) - \max_{j\ne y_i} f_j(\x_i)\\
    &\le f_{y_i}(\x_i) - f_l(\x_i)\\
    &= (\hat{f}_{y_i}^i + \Delta_{y_i}^i) - (\hat{f}_{l}^i + \Delta_{l}^i)\\
    &< \frac{\gamma}{2} + \Delta_{y_i}^i - \Delta_{l}^i\\
    &< \frac{\gamma}{2} + \sup_{\lVert \z \rVert_p < c}(z_1-z_2)\\
    &= \gamma. 
\end{align*}
This shows that $\Id{\mu(\fhat_i, y_i) < \frac{\gamma}{2}} \le \Id{\mu(\f(\x_i), y_i) < \gamma}$, and so we have $R^\frac{\gamma}{2}_{T(\sigma)}(\fhat) \le R^\gamma_{T(\sigma)}(\f)$.

Putting everything back together with the fact that $\epsilon(k)$ is increasing with $k$, as discussed in the proof of the main theorem, we end up with
\begin{equation*}
    R^{\frac{\gamma}{2}}_{T([m+m']-\sigma)}(\fhat) \ge R_{T([m+m']-\sigma)}(\f) > \epsilon(m R^\gamma_{T(\sigma)}(\f)) \ge \epsilon(R^\frac{\gamma}{2}_{T(\sigma)}(\fhat)),
\end{equation*}
which allows us to write
\begin{align*}
\prob{S}{ \exists \f \in \F, R_\D(h_\f) > \epsilon(m R^\gamma_S(\f)) }< 4 \exv{\sigma}{\exv{T}{ \max_{\fhat \in \fcover} \Id{R^{\frac{\gamma}{2}}_{T([m+m']-\sigma)}(\fhat) > \epsilon(m R^\frac{\gamma}{2}_{T(\sigma)}(\fhat)) } }}.
\end{align*}

The remainder of the proof is then identical to the proof of Theorem~\ref{thm:main} with the exception that the empirical risks are replaced by margin risks.
Denoting $\abs{T} \cdot R^{\frac{\gamma}{2}}_T(\f)$ by $k^{\frac{\gamma}{2}}_T(\f)$, it goes as follows.

\begin{align*}
&\prob{S}{ \exists \f \in \F, R_\D(h_\f) > \epsilon(m R^\gamma_S(\f)) }\\
&\hspace{40pt}< 4 \exv{\sigma}{\exv{T}{ \max_{\fhat \in \fcover} \Id{k^{\frac{\gamma}{2}}_{T([m+m']-\sigma)}(\fhat) > m' \epsilon(k^\frac{\gamma}{2}_{T(\sigma)}(\fhat)) } }}\\
&\hspace{40pt}< 4 \exv{\sigma}{\exv{T}{ \sum_{\fhat \in \fcover} \Id{k^{\frac{\gamma}{2}}_{T([m+m']-\sigma)}(\fhat) > m' \epsilon(k^\frac{\gamma}{2}_{T(\sigma)}(\fhat)) } }}\\
&\hspace{40pt}= 4 \exv{\sigma}{\exv{T}{ \sum_{\fhat \in \fcover} \sum_{k=0}^{m-1} \Id{k^{\frac{\gamma}{2}}_{T([m+m']-\sigma)}(\fhat) > m'\epsilon(k) } \Id{k^\frac{\gamma}{2}_{T(\sigma)}(\fhat) = k} }}\\
&\hspace{40pt}= 4 \exv{\sigma}{\exv{T}{ \sum_{\fhat \in \fcover} \sum_{k=0}^{m-1} \Id{k^{\frac{\gamma}{2}}_{T}(\fhat) > m'\epsilon(k) + k } \Id{k^\frac{\gamma}{2}_{T(\sigma)}(\fhat) = k} }}\\
&\hspace{40pt}= 4 \exv{T}{ \sum_{\fhat \in \fcover} \sum_{k=0}^{m-1} \Id{k^{\frac{\gamma}{2}}_{T}(\fhat) > m'\epsilon(k) + k } \hyp\big(k,m, k^\frac{\gamma}{2}_{T}(\fhat), m+m'\big) }\\
\end{align*}
Defining $\kappa(\fhat) \eqdef \max\cb{k : k^{\frac{\gamma}{2}}_{T}(\fhat) > m'\epsilon(k) + k}$, the last expression becomes
\begin{align*}
&\prob{S}{ \exists \f \in \F, R_\D(h_\f) > \epsilon(m R^\gamma_S(\f)) }\\
&\hspace{40pt}< 4 \exv{T}{ \sum_{\fhat \in \fcover} \sum_{k=0}^{\kappa(\fhat)}  \hyp\big(k,m, k^\frac{\gamma}{2}_{T}(\fhat), m+m'\big) }\\
&\hspace{40pt}= 4 \exv{T}{ \sum_{\fhat \in \fcover} \Hyp\big(\kappa(\fhat),m, k^\frac{\gamma}{2}_{T}(\fhat), m+m'\big) }\\
&\hspace{40pt}\le 4 \exv{T}{ \sum_{\fhat \in \fcover} \Hyp\big(\kappa(\fhat),m, \HypInv\big(\kappa(\fhat), \textstyle \frac{\delta}{4 \N_\F\pr{\frac{\gamma}{2^{2 - 1/p}}, m + m'}}, m, m + m'\big), m+m'\big) }\\
&\hspace{40pt}\le 4 \exv{T}{ \sum_{\fhat \in \fcover} \textstyle \frac{\delta}{4 \N_\F\pr{\frac{\gamma}{2^{2 - 1/p}}, m + m'}} }\\
&\hspace{40pt}< \max_T \sum_{\fhat \in \fcover} \textstyle \frac{\delta}{ \N_\F\pr{\frac{\gamma}{2^{2 - 1/p}}, m + m'}}\\
&\hspace{40pt}< \delta.
\end{align*}
This concludes the proof.
\end{proof}

\clearpage
\section{The hypergeometric tail lower bound}
\label{app:lower_bound}

In this appendix, we prove the lower bound, and then we compare it to the upper bound.

\subsection{Proof of the hypergeometric tail lower bound}
\label{app:proof_lower_bound}

In this section, we prove Theorem~\ref{thm:lower_bound}.
For convenience, we reproduce below the statement of the theorem.

\noindent\textbf{Theorem~\ref{thm:lower_bound}}~(Hypergeometric tail inversion lower bound)\textbf{.}
\textit{
Let $\H \subseteq \mathcal{Y}^\mathcal{X}$ be a hypothesis class.
Then for any distribution $\D$ on $\mathcal{X} \times \mathcal{Y}$, any confidence parameter $\delta \in (0, 1)$ and any integers $m, m'>0$, the following inequality holds:
\begin{equation*}
\prob{S\sim\D^m}{\forall h \in \H, R_\D(h) \leq \epsilon(m R_S(h), m, \delta)} > 1 - \delta,
\end{equation*}
where $\epsilon(k,m,\delta) = 0$ if $k=0$ and otherwise
\begin{equation*}
\epsilon(k, m, \delta) = \frac{1}{m'} \min\cb{ m'-1,\; \HypInvLower\big(k-1, m, 1- \textstyle\frac{\delta}{4 \tau_\H(m + m')}, m + m'\big) + 1 - k }.
\end{equation*}
}

The proof being fairly similar to that of Theorem~\ref{thm:main}, we do not give a thorough explanation of steps common to both proofs, and we focus of the differences instead.

\begin{proof}
As in the proof of Theorem~\ref{thm:main}, we shall prove that the opposite event occurs with probability less than or equal to $\delta$, \ie, we want to have
\begin{equation*}
\prob{S}{\exists h \in \H, R_\D(h) < \epsilon(k_S(h))} \le \delta,
\end{equation*}
where we used $k_S(h) = m R_S(h)$.

To apply the ghost sample trick, the result of \citet{greenberg14} used in the proof of the upper bound cannot be used; however, they also provide a similar result which holds for our current needs.
It states that, for all $h\in\H$,
\begin{equation*}
    \prob{S'}{R_{S'}(h) \le R_\D(h)} > \frac{1}{4},
\end{equation*}
as long as $R_\D(h) < \frac{m'-1}{m'}$ holds, a condition satisfied by the definition of $\epsilon$.
Following the same train of thoughts as in the proof of Theorem~\ref{thm:main}, this inequality implies
\begin{align*}
\prob{S}{\exists h \in \H, R_\D(h) < \epsilon(k_S(h))}
&= \exv{S}{\sup_{h\in\H} \Id{R_\D(h) < \epsilon(k_S(h))} }\\
&< 4 \exv{S}{\sup_{h\in\H} \prob{S'}{k_{S'}(h) < m'\epsilon(k_S(h))} }\\
&\le 4 \exv{S,S'}{\sup_{h\in\H} \Id{k_{S'}(h) < m'\epsilon(k_S(h))} }\\
\end{align*}

Next, let $T=(S,S')$ be an i.i.d. sample and let $\H|_T$ be the restriction of $\H$ to $T$.
Using the same notation as before, we have
\begin{align*}
\prob{S}{\exists h \in \H, R_\D(h) < \epsilon(k_S(h))}
&< 4 \exv{T}{\sup_{h\in\H} \Id{k_{T([m+m']-[m])}(h) < m'\epsilon(k_{T([m])}(h))} }\\
&= 4 \exv{T}{\max_{h\in\H|_T} \Id{k_{T}(h) < m'\epsilon(k_{T([m])}(h)) +k_{T([m])}(h) } }\\
&= 4 \exv{T}{\max_{h\in\H|_T} \sum_{k=1}^m \Id{k_{T}(h) < m'\epsilon(k) +k } \Id{k_{T([m])}(h) = k} }\\
&\le 4 \exv{T}{\sum_{h\in\H|_T} \sum_{k=1}^m \Id{k_{T}(h) < m'\epsilon(k) +k } \Id{k_{T([m])}(h) = k} }.
\end{align*}
Note that here the sum over $k$ goes from $1$ to $m$ because the term $k=0$ vanishes since $\epsilon(0)=0$.

For the symmetrization trick, let $\sigma \in \Sigma$ be a combination of $m$ indices chosen among $[m+m']$ as in the proof for the upper bound.
Assuming the uniform distribution over $\Sigma$, we have
\begin{align*}
&\prob{S}{\exists h \in \H, R_\D(h) < \epsilon(k_S(h))}\\
&\hspace{50pt}< 4 \exv{T}{\sum_{h\in\H|_T} \sum_{k=1}^m \Id{k_{T}(h) < m'\epsilon(k) +k } \exv{\sigma}{\Id{k_{T(\sigma)}(h) = k}} }\\
&\hspace{50pt}= 4 \exv{T}{\sum_{h\in\H|_T} \sum_{k=1}^m \Id{k_{T}(h) < m'\epsilon(k) +k } \hyp(k,m,k_T(h),m+m')}.
\end{align*}
Now, define $\kappa(h) \eqdef \min \cb{ k \in [m] : k_{T}(h) < m'\epsilon(k) +k }$.
Because $m'\epsilon(k)+k$ is an increasing function of $k$, this allows us to write
\begin{align}
\prob{S}{\exists h \in \H, R_\D(h) < \epsilon(k_S(h))}\nonumber
&< 4 \exv{T}{\sum_{h\in\H|_T} \sum_{k=\kappa(h)}^m \hyp(k,m,k_T(h),m+m')}\\
&= 4 \exv{T}{\sum_{h\in\H|_T} 1 - \Hyp(\kappa(h)-1,m,k_T(h), m+m')}.\label{eq:lowerbound1}
\end{align}
On the other hand, the definition of $\kappa(h)$ implies that we have (using the definition of $\epsilon$)
\begin{equation*}
k_T(h) \le \HypInvLower\big(\kappa(h)-1, m, \textstyle 1-\frac{\delta}{4 \tau_\H(m+m')}, m+m'\big).
\end{equation*}
Moreover, $\Hyp(k, K, m, M)$ being decreasing with $K$ as shown in Lemma~\ref{lem:monotonicity_hyp_tail}, implies
\begin{align*}
    &1 - \Hyp(\kappa(h)-1, m, k_T(h), m+m')\\
    &\hspace{50pt}\le 1 - \Hyp(\kappa(h)-1, m, \HypInvLower\big(\kappa(h)-1, m, \textstyle 1-\frac{\delta}{4 \tau_\H(m+m')}, m+m'\big), m+m')\\
    &\hspace{50pt}\le 1 - \pr{\textstyle 1-\frac{\delta}{4 \tau_\H(m+m')}}\\
    &\hspace{50pt}= \frac{\delta}{4 \tau_\H(m+m')},
\end{align*}
where the third line holds by definition of the lower pseudo-inverse.

Substituting this result in Equation~\eqref{eq:lowerbound1} yields
\begin{align*}
\prob{S}{\exists h \in \H, R_\D(h) < \epsilon(k_S(h))}
< 4 \exv{T}{\sum_{h\in\H|_T} \frac{\delta}{4 \tau_\H(m+m')}}
< \frac{\delta}{\tau_\H(m+m')} \max_T \Big|\H|_T\Big|
= \delta,
\end{align*}
which concludes the proof.
\end{proof}

\subsection{Numerical comparison of the upper bound with the lower bound}
\label{app:comp_upper_lower}

The lower bound of Theorem~\ref{thm:lower_bound} takes a very similar form to that of the upper bound of Theorem~\ref{thm:main}.
In fact, it is possible to show that they are perfectly symmetric, as we have $\epsilon_\textnormal{upper}(k) = 1 - \epsilon_\textnormal{lower}(m-k)$.
One can observe in Figure~\ref{fig:lower_upper_bounds_comp} that the lower bound is the same as the upper bound reflected simultaneously on both axis.

\begin{figure}[h]
    \centering
    \includegraphics[width=.5\textwidth]{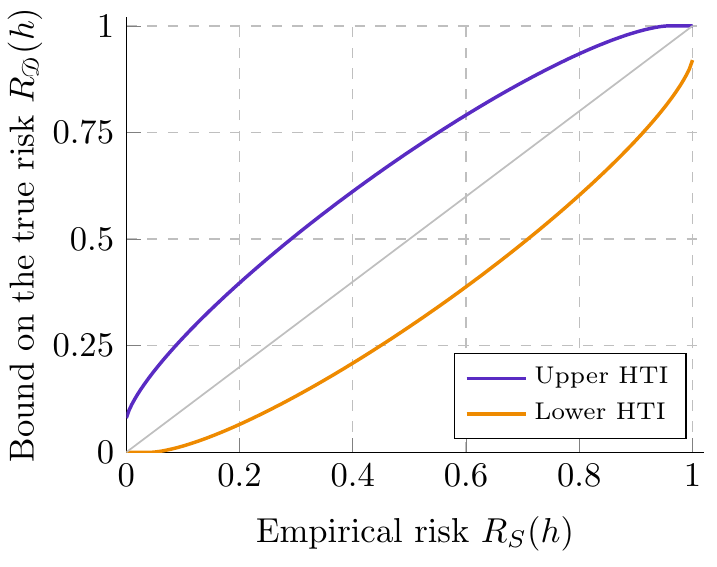}
    \caption{Upper and lower bounds on the true risk as functions of the empirical risk for $m=2000$ examples and a VC dimension of $d=20$. The confidence parameter is $\delta=0.05$.}
    \label{fig:lower_upper_bounds_comp}
\end{figure}

\clearpage

\section{Further numerical comparisons}
\label{app:further_numerical_comparisons}

In this Appendix, we compare the new hypergeometric bound with the other VC bounds for other sets of parameters to give a better overview of the situations for various contexts.
We also discuss sample compression bounds, and how they compare to the new bound.

\subsection{Other sets of parameters}
\label{ssec:other_numerical_settings}

In Section~\ref{sec:numerical_comparison}, we have compared the new hypergeometric tail inversion (HTI) bound to Catoni's Theorem 4.6 (C4.6), to Vapnik's pessimistic (VP) bound, to Vapnik's relative deviation (VRD) bound, and to Lugosi's chaining bound, as functions of the empirical risk, the VC dimension and the sample size.
While the parameters chosen to plot each bound in these figures may represent a typical machine learning problem, and coincidentally allows to give a good view of all the bounds, it is not necessarily representative of all possible situations.
In this section, we reproduce Figures~\ref{fig:bounds_comp_risk} and \ref{fig:bounds_comparison_m} for other values of $m$ and $R_S(h)$ in order to have a more complete view of how the new HTI bound compares in other settings.

Figure~\ref{fig:bounds_comp_risks_app} reproduces Figure~\ref{fig:bounds_comp_risk} for values of $m$ equal to $100$, $500$, $2000$ and $20$,$000$, keeping $d=50$ and $\delta=0.05$.
The hyperparameter $m'$ is optimized for $R_S(h)=0$ in all four cases.
Lugosi's chaining bound is vacuous and does not appear in any of the plots.

\begin{figure}[p]
\centering
\begin{subfigure}[t]{0.485\textwidth}
    \centering
    \includegraphics[width=\textwidth]{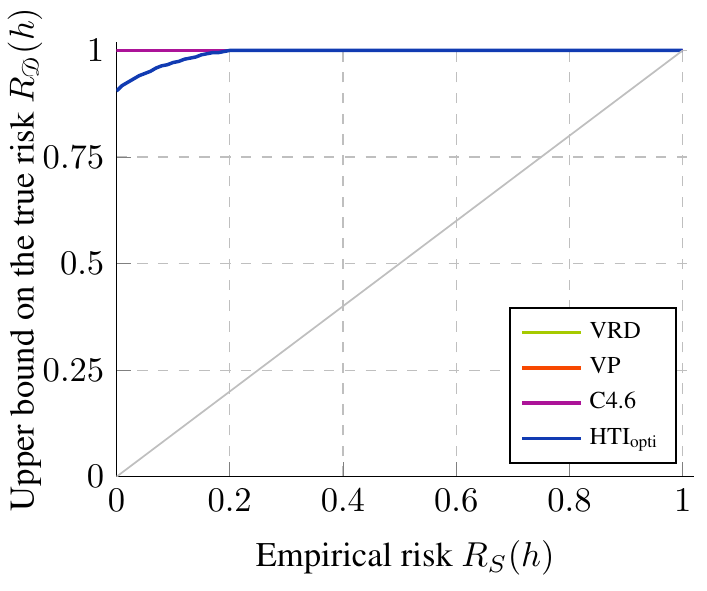}
    \caption{Bounds as functions of the empirical risk for $m=100$ and $d=50$.}
    \label{fig:bounds_comp_risk_m=100_d=50}
\end{subfigure}\hfill
\begin{subfigure}[t]{0.485\textwidth}
    \centering
    \includegraphics[width=\textwidth]{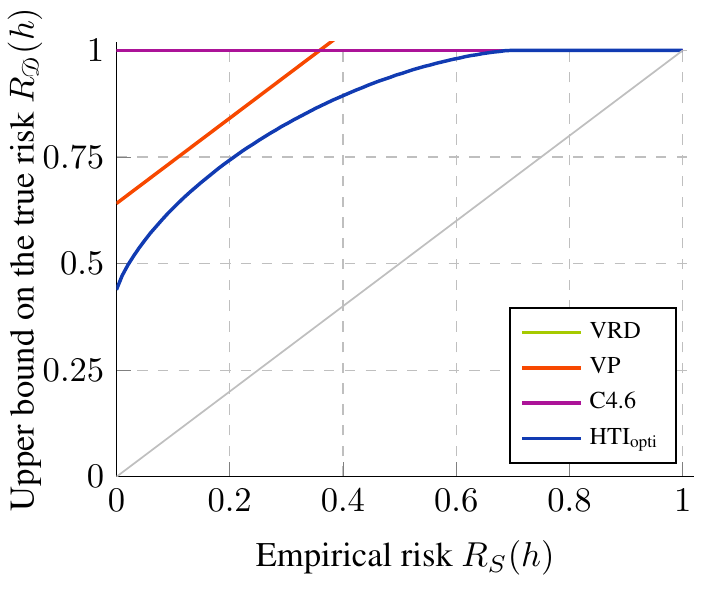}
    \caption{Bounds as functions of the empirical risk for $m=500$ and $d=50$.}
    \label{fig:bounds_comp_risk_m=500_d=50}
\end{subfigure}

\begin{subfigure}[t]{0.485\textwidth}
    \centering
    \includegraphics[width=\textwidth]{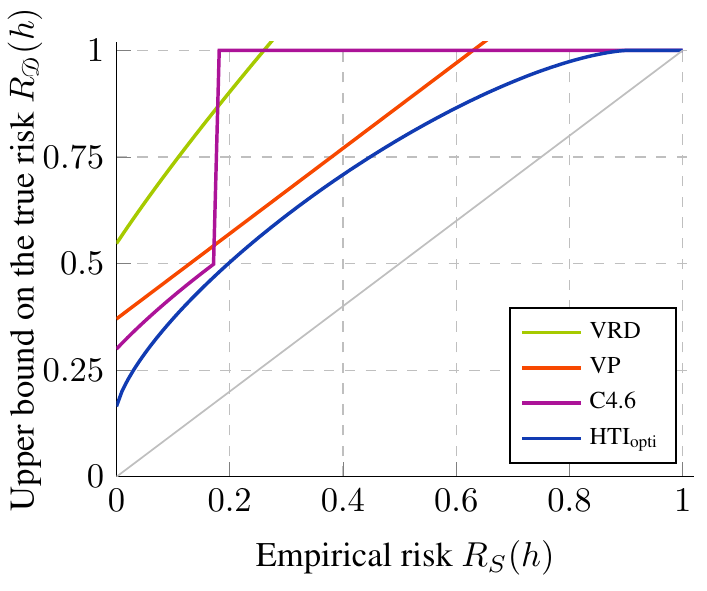}
    \caption{Bounds as functions of the empirical risk for $m=2000$ and $d=50$.}
    \label{fig:bounds_comp_risk_m=2000_d=50}
\end{subfigure}\hfill
\begin{subfigure}[t]{0.485\textwidth}
    \centering
    \includegraphics[width=\textwidth]{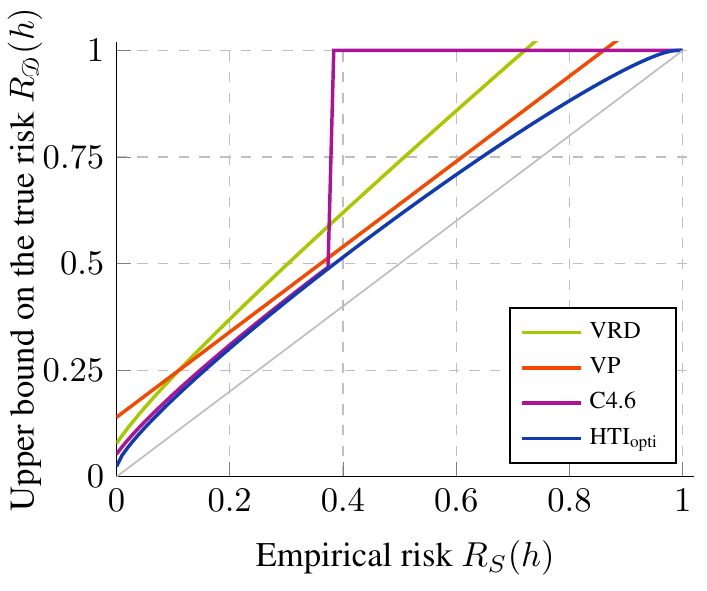}
    \caption{Bounds as functions of the empirical risk for $m=20$,$000$ and $d=50$.}
    \label{fig:bounds_comp_risk_m=20000_d=50}
\end{subfigure}
\caption{The bounds as functions of the empirical risk for multiple sample sizes $m$.
The confidence parameter is $\delta=0.05$ and the VC dimension is $d=50$.}
\label{fig:bounds_comp_risks_app}
\end{figure}

We note that when $m=100$, which is as small as $2d$, only the new HTI bound is non-vacuous, and for $m=500=10d$, C4.6 is still vacuous.
Figure~\ref{fig:bounds_comp_risk_m=2000_d=50}, with $m=2000$, gives a clearer picture of the situation, where all bounds are non-vacuous for low risks and are visually distinct.
Finally, for a large sample size of $m=20$,$000$ examples, all bounds seem to clump up, but the HTI bound is still always tighter than the other bounds.
In particular, at $R_S(h)=0$, the HTI bound takes a value of $0.024$ while C4.6 takes a value of $0.053$---a relative gain of 54.5\%.

Figure~\ref{fig:bounds_comp_ms_app} reproduces Figure~\ref{fig:bounds_comparison_m} for values of $R_S(h)$ equal to $0$, $0.1$ in addition to $0.2$ and $0.3$, again keeping $d=50$ and $\delta=0.05$.
The hyperparameter $m'$ is optimized individually for each value of $m$ and $R_S(h)$.

\begin{figure}[p]
\centering
\begin{subfigure}[t]{0.485\textwidth}
    \centering
    \includegraphics[width=\textwidth]{figures/bounds_comparison_m_risk=0_d=50.pdf}
    \caption{Bounds as functions of the sample size for $R_S(h)=0$ and $d=50$.}
    \label{fig:bounds_comp_m_risk=0_d=50_app}
\end{subfigure}\hfill
\begin{subfigure}[t]{0.485\textwidth}
    \centering
    \includegraphics[width=\textwidth]{figures/bounds_comparison_m_risk=0.1_d=50.pdf}
    \caption{Bounds as functions of the sample size for $R_S(h)=0.1$ and $d=50$.}
    \label{fig:bounds_comp_m_risk=0.1_d=50_app}
\end{subfigure}

\begin{subfigure}[t]{0.485\textwidth}
    \centering
    \includegraphics[width=\textwidth]{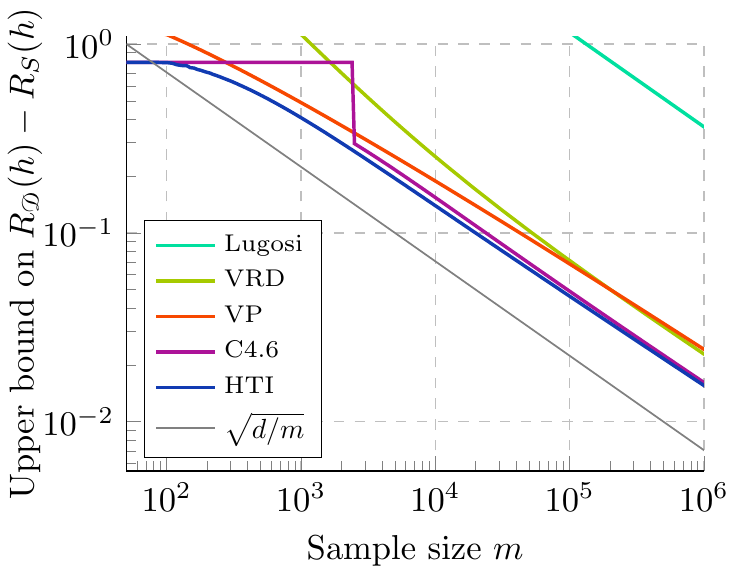}
    \caption{Bounds as functions of the sample size for $R_S(h)=0.2$ and $d=50$.}
    \label{fig:bounds_comp_m_risk=0.2_d=50}
\end{subfigure}\hfill
\begin{subfigure}[t]{0.485\textwidth}
    \centering
    \includegraphics[width=\textwidth]{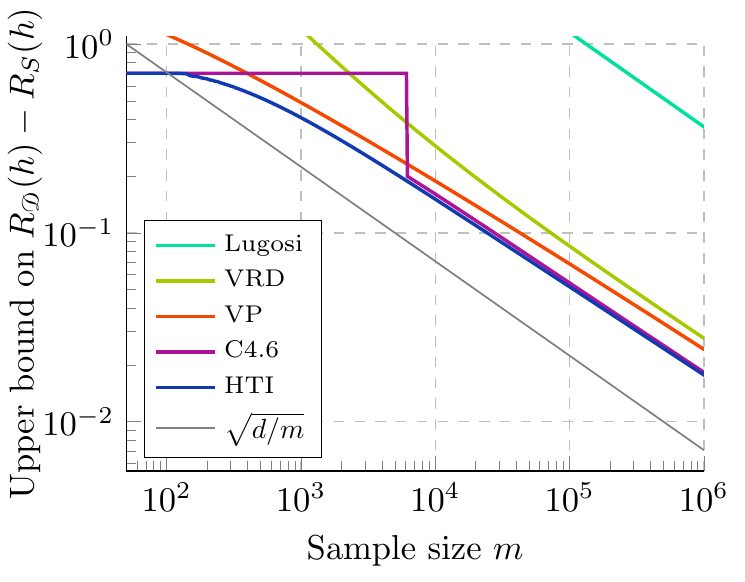}
    \caption{Bounds as functions of the sample size for $R_S(h)=0.3$ and $d=50$.}
    \label{fig:bounds_comp_m_risk=0.3_d=50}
\end{subfigure}
\caption{The bounds as functions of the sample size for multiple empirical risks.
The confidence parameter is $\delta=0.05$ and the VC dimension is $d=50$.}
\label{fig:bounds_comp_ms_app}
\end{figure}

Figure~\ref{fig:bounds_comp_m_risk=0_d=50_app} and \ref{fig:bounds_comp_m_risk=0.1_d=50_app} are identical to the figures shown in the body of the paper.
However, we can see from Figures~\ref{fig:bounds_comp_m_risk=0.2_d=50} and \ref{fig:bounds_comp_m_risk=0.3_d=50} that increasing the empirical risks seems to worsen the asymptotic rates of all bound except the chaining bound.
Hence, we can observe that the effect of the extra unnecessary log terms in the bounds matters more when the empirical risk is large.

For all the previously considered settings, we have not varied the VC dimension $d$, which was kept equal to $50$.
However, since all bounds are asymptotic functions of $\frac{d}{m}$, increasing $d$ has roughly the same effect on all bounds as decreasing $m$ in the same proportions (and vice versa).

Indeed, Figure~\ref{fig:bounds_comp_d=10_app} plots the bounds as a function of the empirical risk and the sample size for $d$ to $10$.
In Figure~\ref{fig:bounds_comp_risk_m=400_d=10}, the sample size is set to $m=400$, a five-fold decrease compared to the sample size of Figure~\ref{fig:bounds_comp_risk}.
Comparing both figures, one can see that they are nearly identical, confirming our claim.

Figure~\ref{fig:bounds_comp_m_risk=0.1_d=10} on the other hand shows the asymptotic rate of the bounds for $d=10$ and an empirical risk of $0.1$.
One can see that the curves are extremely similar to what is observed in Figure~\ref{fig:bounds_comp_m_risk=0.1_d=50}, with the exception that the $x$-axis is shifted by a constant (due to the log scale), as one would expect.

\begin{figure}[h!]
\centering
\begin{subfigure}[t]{0.485\textwidth}
    \centering
    \includegraphics[width=\textwidth]{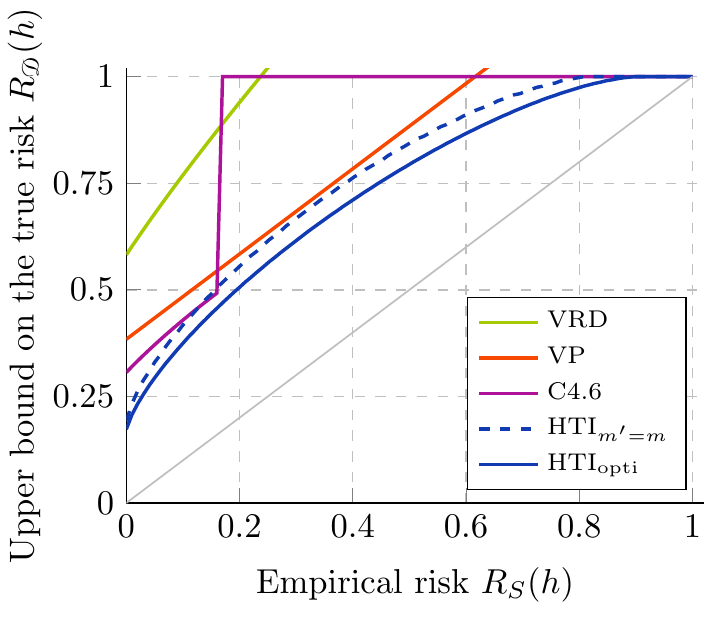}
    \caption{Bounds as functions of the empirical risk for $m=400$ and $d=10$.}
    \label{fig:bounds_comp_risk_m=400_d=10}
\end{subfigure}\hfill
\begin{subfigure}[t]{0.485\textwidth}
    \centering
    \includegraphics[width=\textwidth]{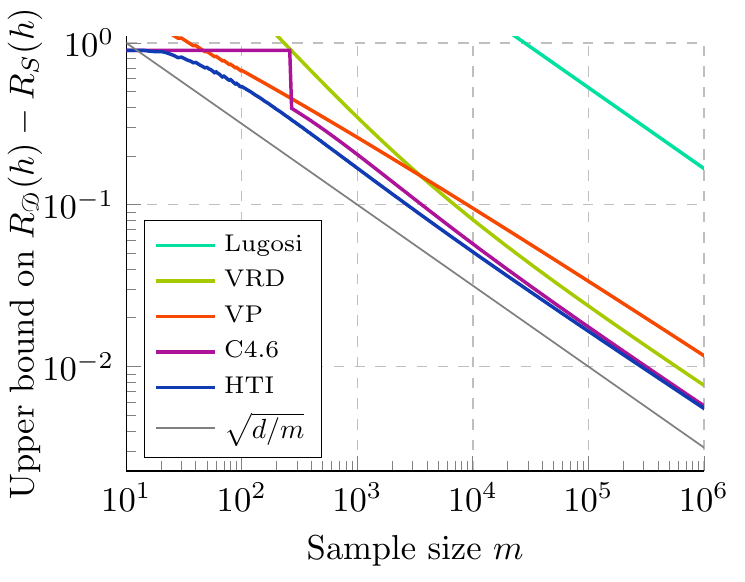}
    \caption{Bounds as functions of the sample size for $R_S(h)=0.1$ and $d=10$.}
    \label{fig:bounds_comp_m_risk=0.1_d=10}
\end{subfigure}
\caption{Comparison of the bounds.
Left: the bounds as functions of the empirical risk for $m=400$. Right: the bounds on $R_\D(h) - R_S(h)$ as functions of the sample size $m$ for $R_S(h) = 0.1$.
The confidence parameter is $\delta=0.05$ and the VC dimension is $d=10$ for both cases.}
\label{fig:bounds_comp_d=10_app}
\end{figure}



We have not yet discussed the confidence parameter $\delta$.
The HTI, VP, VRD and C4.6 bounds are all functions of $\log(\tau_\H(m+m')/\delta)$.
Therefore, a linear decrease in $\delta$ is roughly equivalent to a logarithmic increase of the VC dimension (and vice versa).
As for Lugosi's chaining bound, it depends on $\delta$ via an additive factor of $\log(1/\delta)$.
Consequently, the variation of $\delta$ has marginal effects on all of the bounds compared to the other three parameters.


\subsection{Comparison to sample compression}
\label{ssec:sc_comparaison}

Sample compression bounds~\citep{floyd95sample} exploit the idea that a classifier can be described with only a small subset of the sample.
Such a subset is called a compression set, and its size $d$ acts as an indicator of the complexity of the chosen class of hypotheses.
In that sense, this parameter is analogous to the VC dimension.
Sample compression bounds are known to be very tight (although subject to the same optimal asymptotic rate as VC bounds).
Indeed, using the binomial tail inversion, \citet{langford05} and \citet{Laviolette09} have shown that they are nearly as tight as one might hope.

The Warmuth-Littlestone~\citep{floyd95sample} conjecture states that any hypothesis class with VC dimension $d$ admits a compression scheme with size in $\O(d)$.%
\footnote{A strong version of the conjecture, stated by \citet{kuzmin2007unlabeled}, requires a compression size of at most $d$, which has been refuted by \cite{palvolgyi2020unlabeled}, although this result concerns unlabeled compression schemes.}
If this conjecture was true, it would imply that one could use the (tighter) sample compression bounds instead of traditional VC bounds for most applications.
Yet, despite some significant progress \citep{moran16sample}, the weak version of the Warmuth-Littlestone conjecture is still an open problem.
Furthermore, it is not always clear how to construct an appropriate compression scheme with a reasonable message length, which has a compression size equal to (or slightly larger than) the VC dimension of the hypothesis class.
Therefore, there are still plenty of situations where VC bounds can prove useful.

Nevertheless, it is of interest to see how the new VC bound compares to sample compression bounds.
Thus, we assume a context where both generalization bounds apply, and we consider an \emph{optimistic} situation where no reconstruction message is required, since it would be specific to the hypothesis class.
Note that such a reconstruction message would loosen the compression bound.
With these assumptions, we can use the following bound \citep{langford05, Laviolette09}:
\begin{equation*}
    \epsilon_\textnormal{SC}(R_S(h)) = \BinInv\pr{mR_S(h), m-d,\textstyle \frac{\delta}{m \binom{m}{d}}},
\end{equation*}
where $\BinInv$ is the binomial tail inversion (defined analogously to $\HypInv$) and $d$ is the size of the compression scheme.
Assuming the Warmuth-Littlestone conjecture holds, we set the compressed sample size equal to the VC dimension at a value of $50$.
Figure~\ref{fig:sc_comparison} reproduces the comparison of Figures~\ref{fig:bounds_comp_risk} and \ref{fig:bounds_comp_m_risk=0.1_d=50}, but for the sample compression (SC) bound and the new hypergeometric tail inversion (HTI) bound.

\begin{figure}[h]
\centering
\begin{subfigure}[t]{0.485\textwidth}
    \centering
    \includegraphics[width=\textwidth]{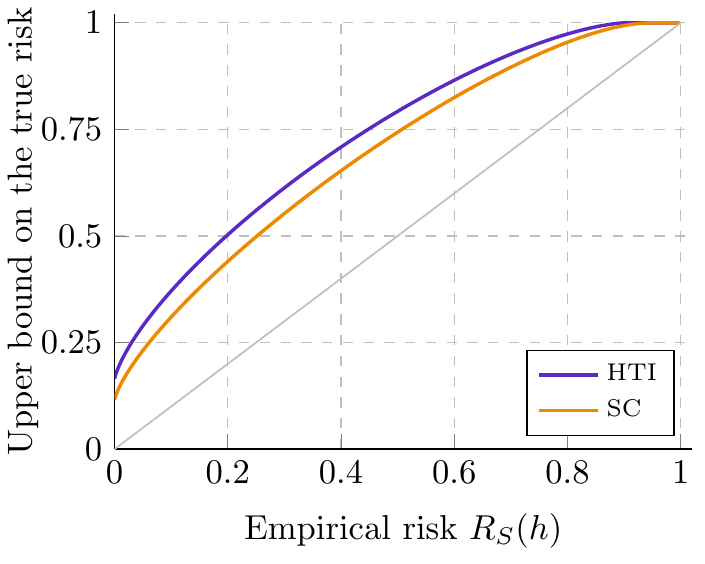}
    \caption{Bounds as functions of the empirical risk. The gray line corresponds to the empirical risk.}
    \label{fig:sc_comp_risk}
\end{subfigure}\hfill
\begin{subfigure}[t]{0.485\textwidth}
    \centering
    \includegraphics[width=\textwidth]{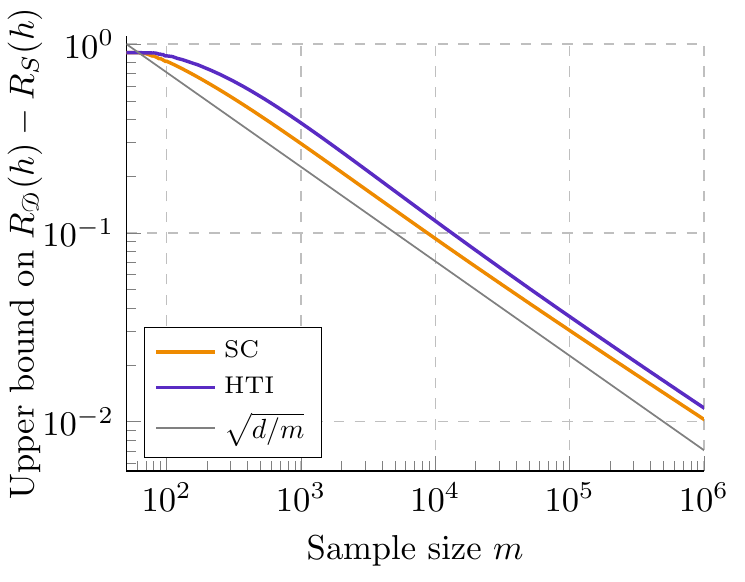}
    \caption{Bounds on the risks' difference as functions of the sample size. The gray line corresponds to the optimal asymptotic behavior.}
    \label{fig:sc_comp_m}
\end{subfigure}
\caption{Comparison of the bounds.
Left: the bounds as functions of the empirical risk for $m=2000$. Right: the bounds on $R_\D(h) - R_S(h)$ as functions of the sample size $m$ for $R_S(h) = 0.1$.
The confidence parameter is set to $\delta=0.05$. Both the VC dimension and the compressed sample size are set to $d=50$.}
\label{fig:sc_comparison}
\end{figure}

We can see that the sample compression bound is always slightly tighter than the HTI bound.
However, it is interesting to see that both bounds have a very similar ``arc'' shape which dips for risks near 0---a phenomenon absent in other VC bounds.
Hence, even though the sample compression bound may not be always valid for VC classes, the HTI bound is very close to it and has a similar behaviour.

\vskip 0.2in
\bibliography{references}

\end{document}